\documentclass{article}

\usepackage{microtype}
\usepackage{graphicx}
\usepackage{subfigure}
\usepackage{booktabs} % for professional tables

\usepackage{hyperref}

\usepackage[accepted]{icml2024}

\usepackage{amsmath}
\usepackage{amssymb}
\usepackage{mathtools}
\usepackage{amsthm}

\usepackage{multirow}
\usepackage{dsfont}
\usepackage{tablefootnote}

\usepackage{color} 

\definecolor{airforceblue}{rgb}{0.36, 0.54, 0.66}

\usepackage[capitalize,noabbrev]{cleveref}

\theoremstyle{plain}
\newtheorem{theorem}{Theorem}[section]
\newtheorem{proposition}[theorem]{Proposition}
\newtheorem{lemma}[theorem]{Lemma}

\theoremstyle{definition}
\newtheorem{definition}[theorem]{Definition}
\newtheorem{assumption}[theorem]{Assumption}
\theoremstyle{remark}
\newtheorem{remark}[theorem]{Remark}

\usepackage[textsize=tiny]{todonotes}

\icmltitlerunning{Stability and Generalization of D-SGD}

\newcommand{\calD}{{\cal D}}

\newcommand\calE{{\cal E}}

\newcommand\calO{{\cal O}}

\newcommand\calZ{{\cal Z}}
\newcommand\calU{{\cal U}}

\newcommand\IR{{\mathbb{R}}}

\newcommand\IP{{\mathbb{P}}}

\newcommand{\eg}            {e.g.\ }

\DeclareMathOperator*{\argmin}{arg\,min}

\newcommand{\RR}{\mathbb{R}} % espace réel

\newcommand{\TS}{\Tilde{S}}

\newcommand{\TZ}{\Tilde{Z}}
\newcommand{\Ttheta}{\Tilde{\theta}}

\newcommand{\EE}{\mathbb{E}}

\newcommand{\1}{\mathbf{1}}

\begin{document}

\twocolumn[
\icmltitle{Improved Stability and Generalization
Guarantees \\ of the Decentralized SGD Algorithm }

%\icmltitle{On the Impact of the Graph for Stability and \\ Generalization of Decentralized SGD}



\icmlsetsymbol{equal}{*}

\begin{icmlauthorlist}
\icmlauthor{Batiste Le Bars}{yyy}
\icmlauthor{Aurélien Bellet}{sch}
\icmlauthor{Marc Tommasi}{comp}
\icmlauthor{Kevin Scaman}{yyy}
\icmlauthor{Giovanni Neglia}{gio}
%\icmlauthor{Firstname8 Lastname8}{sch}
%\icmlauthor{Firstname8 Lastname8}{yyy,comp}
\end{icmlauthorlist}

\icmlaffiliation{yyy}{Inria Paris - Ecole Normale Supérieure, PSL Research University}
\icmlaffiliation{comp}{Univ. Lille, Inria, CNRS, Centrale Lille, UMR 9189, CRIStAL, F-59000 Lille}
\icmlaffiliation{sch}{Inria, Université de Montpellier}
\icmlaffiliation{gio}{Inria, Université Côte d’Azur}

\icmlcorrespondingauthor{Batiste Le Bars}{batiste.le-bars@inria.fr}

\icmlkeywords{Learning theory, Decentralized optimization}

\vskip 0.3in
]



\printAffiliationsAndNotice{}  % leave blank if no need to mention equal contribution

\begin{abstract}
    This paper presents a new generalization error analysis for
    Decentralized Stochastic Gradient Descent (D-SGD) based on
    algorithmic stability. The obtained results overhaul a series of recent works that suggested an increased instability due to decentralization and a detrimental impact of poorly-connected communication graphs on generalization. On the contrary, we show, for convex, strongly convex and non-convex functions, that D-SGD can always recover generalization bounds analogous to those of classical SGD, suggesting that the choice of graph does not matter. We then argue that this result is coming from a worst-case analysis, and we provide a refined optimization-dependent generalization bound for general convex functions. This new bound reveals that the choice of graph can in fact improve the worst-case bound in certain regimes, and that surprisingly, a poorly-connected graph can even be beneficial for generalization.

\end{abstract}

%\bat{Idée titre: On the real impact of the graph for stability and generalization of the decentralized SGD alorithm }

% !TEX root = ../neurips_2023.tex

\section{Introduction}

Studying the ability of machine learning models to generalize to unseen data is a
fundamental and long-standing objective. Among the several approaches that have
been proposed to bound generalization errors, the most prominent ones are
based on the complexity of the hypothesis class like the Vapnik-Chervonenkis
dimension or Rademacher complexity \citep{bousquet2004introduction},
algorithmic stability \citep{bousquet2002stability}, PAC-Bayesian bounds \citep{ShaweTaylor97,mcallester98,catoni2007pac, alquier2021user}, or more recently information-theoretic generalization bounds \citep{xu2017information}.

Over the last few years, a substantial amount of work has been dedicated to
 the study of the generalization properties of \emph{optimization algorithms},
 more specifically gradient-based methods \citep{lin2016generalization,
 london2017pac, zhou2018generalization,amir2021sgd, neu2021information, scaman2024minimax}. In
 particular, since the seminal work of \citet{hardt2016train}, approaches
 based on \emph{algorithmic stability} have encountered a large success as
 they allow to shed light on the implicit regularization brought by 
(stochastic) gradient methods \citep{kuzborskij2018data,bassily2020stability,lei2020fine,schliserman2022stability}. However, this large amount of work is mostly focusing on \emph{centralized} gradient-based algorithms.
% and analyses for \emph{decentralized} algorithms is lacking.
%, and the generalization properties of \emph{decentralized} gradient methods is far less understood.

Decentralized learning algorithms, such as the celebrated Decentralized
Stochastic Gradient Descent (D-SGD) algorithm \citep{nedic2009distributed}, allow
several agents to train models on their local data by exchanging
model updates rather than the data itself. In D-SGD, agents
solve an empirical risk minimization task by alternating between computing local
gradient steps and averaging model parameters with their neighbors in a
communication graph.
%, which encodes which pairs of agents exchange information.
A sparser graph thus reduces the per-iteration
communication cost but tends to increase the number of iterations needed to
converge. Most theoretical analyses of D-SGD and its variants focus on
understanding the
\emph{optimization error} by
characterizing the convergence rate to the empirical risk minimizer. They 
notably highlight the impact of the communication graph and data heterogeneity
 across agents
\citep{koloskova2020unified,neglia2020decentralized,ying2021exponential,le2023refined}. In contrast,
the \emph{generalization error} of decentralized learning algorithms is far
less understood.

In the work of \citet{richards2020graph}, the authors focus on a specific variant of D-SGD (thereafter referred to as Variant~A, see Algorithm~\ref{alg:d-sgd}) where the agents perform a local stochastic gradient update \emph{before} averaging their parameters with their neighbors. With an analysis based on algorithmic stability, they come to the conclusion that for convex functions the decentralization does not have any impact, recovering the same generalization bounds as those obtained by \citet{hardt2016train} for centralized SGD. In contrast, a more recent line of work \citep{sun2021stability, zhu2022topology, taheri2023generalization} has investigated this question for a more practically relevant variant of D-SGD (thereafter referred to as Variant~B, see Algorithm \ref{alg:d-sgd}), where this time the averaging step and the local gradient computation are done \emph{in parallel}. % This variant can be faster in practice, % is supposedly more adapted to the decentralized setting,
% but is also more difficult to analyze \citep[Remark 4]{richards2020graph}.
The existing generalization results for this variant, which is known to be more difficult to analyze \citep[see][Remark~4 therein]{richards2020graph}, mainly differ in their
technical
assumptions. \citet{sun2021stability} consider Lipschitz and smooth loss
functions, \citet{zhu2022topology} focus on smooth and convex losses, while 
\citet{taheri2023generalization} investigate the overparameterized regime, for convex and Polyak-Lojasiewicz functions. Despite these differences, these three studies all
come to the same conclusion: ``decentralization has a negative impact on
generalization.'' Specifically, their generalization bounds get
larger as the graph gets sparser, and eventually become vacuous for non-connected
graphs. %Worse still, the generalization gap does not even tend to zero as the
%number of local samples increases. 
%This is quite surprising as the two variants of D-SGD, recalled in Section \ref{sec:dsgd}, have the same optimization performances, and there is no reason for them to be so different in terms of generalization.
%The generalization properties of the two  variants of D-SGD would then be in strike contrast. This result is surprising as the two exhibit minor differences (see Section \ref{sec:dsgd}) and for example have the same optimization performances.
In summary, the current literature exhibits strikingly contrasting generalization properties for these two common D-SGD variants, despite their similar optimization performance. While questioning this gap is the main motivation for our work, we also wonder how generalization upper bounds like those obtained by \citet{richards2020graph} can be completely independent of the communication graph. This result is indeed rather counter-intuitive, as we know how important the choice of communication graph can be for optimization \cite{neglia2020decentralized}. \looseness = -1
% , despite their few differences and identical optimization performances, as detailed in Section \ref{sec:dsgd}. This result is very surprising.

\subsection{Contributions}

In this work, we focus on the more complex Variant~B and prove that the dichotomy between the two variants of D-SGD is only apparent. We show that they are in fact equivalent in terms of stability and generalization performance, improving upon the recent conclusions of \citet{sun2021stability, zhu2022topology} and \citet{taheri2023generalization}. Our contributions, summarized in Table \ref{tab:a}, are the following:

\textbf{(1)} In Section \ref{sec:convex}, we first consider convex and strongly convex loss functions and show that we can recover, for Variant~B of D-SGD, the exact same generalization upper-bounds than those of Variant~A \citep{richards2020graph} and standard SGD \citep{hardt2016train}. This leads to the conclusion that, contrary to the optimization error, the choice of graph and in particular \textbf{poorly-connected graphs do not have a detrimental impact on generalization}.

\textbf{(2)} We then consider in Section \ref{sec:non-convex} the less studied case of non-convex functions, which, contrary to convex ones, were not considered by \citet{richards2020graph} for Variant~A of D-SGD. Again, we show that, for both variants, it is possible to recover almost identical generalization upper-bounds as those obtained by \citet{hardt2016train} for SGD, %improving the result obtained for Variant~B by \cite{sun2021stability}. This 
leading to the same conclusions as for convex cases regarding the lack of impact of decentralization on generalization. 

\textbf{(3)} We finally argue in Section \ref{sec:data-dep} that our analysis, as well as the one of \citet{richards2020graph}, characterize ``worst-case'' generalization guarantees across many possible losses and data distributions. One may then wonder if the communication graph can play a role under more specific losses or distributions. To address this point, we propose a refined analysis for convex functions, inspired by optimization-dependent generalization bounds for classical SGD \citep{kuzborskij2018data,lei2020fine}, which confirms this is indeed the case. %This new analysis expresses the generalization bound as a function of the optimization trajectory and is inspired by generalization bounds for classical SGD from \citet{kuzborskij2018data} and \citet{lei2020fine}. 
Quite surprisingly, our new bound not only shows that in low-noise regimes specific choice of graphs can improve the worst-case bound, but also that \textbf{poorly-connected graphs can even be beneficial to generalization}.

%\marc{Faudrait peut-être dire un mot
%sur les techniques de preuve, mais bon elles sont essentiellement celles de Hardt... }
%\aurelien{éventuellement parler de l'histoire de consensus distance qu'on
%évite; sinon ne rien dire}

Before moving to our main contributions, the following section provides relevant background on the relationship between algorithmic stability and generalization, presents D-SGD, and discusses the main assumptions considered throughout the paper. %\looseness = -1

\begin{table}[]
\setlength\tabcolsep{2.8pt}
\caption{Simplified generalization bounds for (D)-SGD with Lipschitz and
    smooth loss functions. [H] indicates the results of 
    \citet{hardt2016train}, [R] those of \citet{richards2020graph}, and [S] those of \citet{sun2021stability}. For
    simplicity, we omit constant factors. $T$ is the number of iterations, $m$ the number of agents, $n$ the number of local data points, and $\rho\in[0,1]$ the spectral gap of the communication graph. We also have $a \in (0,1)$ a constant that depends on the model parameters and $C_\rho$ a constant that depends on~$\rho$. For centralized SGD, we consider that the algorithm is run over $mn$ data points. We refer to Section \ref{sec:dsgd} for the definitions of Variants~A and B of D-SGD.} 

\vspace{0.3cm}
\begin{tabular}{|c||c||c|c|}
\hline
\multirow{2}{*}{}      & \multirow{2}{*}{SGD} & \multicolumn{2}{c|}{D-SGD}                                                           \\ \cline{3-4} 
                       &                      & \multicolumn{1}{c|}{Variant A} & Variant B                                           \\ \hline
Convex                & $\frac{T}{mn}$ [H]                   & \multicolumn{1}{c|}{$\frac{T}{mn}$ [R]}      & \begin{tabular}[c|]{@{}c@{}}$\frac{T}{mn}+\frac{T}{\rho}$ [S] \\[0.2cm]  $\frac{T}{mn}$ [\textbf{ours}]\end{tabular}        \\ \hline
\begin{tabular}[c]{@{}c@{}}$\mu$-Strongly\\  convex\end{tabular} & $\frac{1}{\mu mn}$ [H]                  & \multicolumn{1}{c|}{$\frac{1}{\mu mn}$ [R]}         & \begin{tabular}[c|]{@{}c@{}}$\frac{1}{\mu mn}+\frac{1}{\mu\rho}$ [S]  \\[0.2cm]  $\frac{1}{\mu mn}$ [\textbf{ours}]\end{tabular}  \\ \hline
Non-convex            &  $\frac{T^{a}}{mn}$ [H]                       & \multicolumn{1}{c|}{ $\frac{T^{a}}{m^{1-a}n}$ [\textbf{ours}]} & \begin{tabular}[c|]{@{}c@{}}$\frac{T^{a}}{n} + C_\rho T^{a}$ [S]\tablefootnote{In the original paper \citep{sun2021stability}, the authors give a bound in $\calO(\frac{T^{a}}{mn} + C_\rho T^{a})$. In Appendix \ref{app:mistake_sun}, we reveal that there is a mistake in their proof, corrected in Table \ref{tab:a}. } \\[0.2cm]   $\frac{T^{a}}{m^{1-a}n}$ [\textbf{ours}]\end{tabular} \\ \hline
\end{tabular}
\label{tab:a}
\end{table}

%%% Local Variables:
%%% mode: latex
%%% TeX-master: "../neurips_2023"
%%% End:

% !TEX root = ../neurips_2023.tex

\section{Background}

\subsection{Stability and Generalization in Decentralized Learning}

\label{sec:background}

We consider the general setting of statistical learning, adapted to a decentralized framework with $m$ agents. We consider that agent $k$  observes examples drawn from
a local data distribution $\calD_k$ with support $\calZ$. The objective is to find a global model $\theta \in \RR^d$ minimizing the \emph{population risk} defined by:
$$R(\theta)\triangleq \frac{1}{m}\sum_{k=1}^m\EE_{Z \sim \calD_k} [\ell(\theta; Z)]\;,$$
where $\ell$ is some loss function. We denote by $\theta^\star$ a global minimizer of the population risk, i.e., $\theta^\star\in \argmin_ \theta R(\theta)$.

Although we cannot evaluate the population risk $R(\theta)$, we have
access to an empirical counterpart, computed over $m$ local datasets
$S\triangleq (S_1,\ldots,S_m)$ where $S_k = \{Z_{1k},\ldots,Z_{nk}\}$ is the
dataset of agent $k$ with $Z_{ik}\sim \calD_k$. Note that for simplicity we
consider that all local datasets are of same size~$n$, but our analysis can
be extended to the heterogeneous case. The resulting \emph{empirical risk} is given by:
$$R_S(\theta)\triangleq \frac{1}{m}\sum_{k=1}^m R_{S_k}(\theta)\triangleq \frac{1}{mn}\sum_{k=1}^m\sum_{i=1}^n \ell(\theta; Z_{ik})\;.$$

One of the most famous and studied estimators is the empirical risk minimizer,
denoted by $\widehat{\theta}_{\text{ERM}} \triangleq \argmin_\theta R_S(\theta)$. However, in most situations, this estimator cannot be directly computed. Instead, one relies on a potentially random \emph{decentralized optimization} algorithm $A$, taking as input the full dataset $S$, and returning an approximate minimizer $A(S) \in \RR^d$ of the empirical risk $R_S(\theta)$. 

In this setting, we can upper-bound the expected \emph{excess risk} $R(A(S)) -
R(\theta^\star)$ by the sum of the (expected) 
\emph{generalization} error ($\epsilon_{\textrm{gen}}$), and the (expected) \emph{optimization error} ($\epsilon_{\textrm{opt}}$):
% \begin{align*}
%     \EE_{A,S}& [R(A(S)) - R(\theta^\star)] \\
%     & = \EE_{A,S}[R(A(S)) - R_S(A(S))] + \EE_{A,S}[R_S(A(S)) - R_S(\widehat{\theta}_{\text{ERM}})] + \EE_{A,S}[R_S(\widehat{\theta}_{\text{ERM}}) - R(\theta^\star)] \\
%     & = \epsilon_{\textrm{gen}} + \epsilon_{opt} + \EE_{A,S}[\underbrace{R_S(\widehat{\theta}_{\text{ERM}}) - R_S(\theta^\star)}_{\leq 0}]  \leq \epsilon_{\textrm{gen}} + \epsilon_{opt}
% \end{align*}
\begin{equation*}
    \EE_{A,S}  [R(A(S)) - R(\theta^\star)] \leq \epsilon_{\textrm{gen}} + \epsilon_{\textrm{opt}}\;, %\\
    %& \hspace{-0.5cm}  = \underbrace{\EE_{A,S}[R(A(S)) - R_S(A(S))]}_{\epsilon_{\textrm{gen}}} + \underbrace{\EE_{A,S}[R_S(A(S)) - R_S(\widehat{\theta}_{\text{ERM}})]}_{\epsilon_{\textrm{opt}}} + \underbrace{\EE_{A,S}[R_S(\widehat{\theta}_{\text{ERM}}) - R(\theta^\star)]}_{\leq 0}\;. %\\
    %&  \hspace{-0.5cm} \leq \epsilon_{\textrm{gen}} + \epsilon_{opt}
\end{equation*}
where $\epsilon_{\textrm{gen}} \triangleq \EE_{A,S}[R(A(S)) - R_S(A(S))]$ and $\epsilon_{\textrm{opt}} \triangleq \EE_{A,S}[R_S(A(S)) - R_S(\widehat{\theta}_{\text{ERM}})]$.
The present work focuses on the control of the expected generalization
error $\epsilon_{\textrm{gen}}$, for which a popular approach is based on the stability analysis of the algorithm $A$.\footnote{While we focus here on the \emph{expected} version of the generalization error, some of these tools are also well-suited to provide \emph{high-probability} generalization bounds \citep{feldman2019high}.} 

% Below, we recall the notion of \emph{uniform stability} that is widely considered for this kind of analysis.

% \begin{definition}\emph{(Uniform Stability).}
% \label{def:uniform_stab}
%     A randomized algorithm $A$ is $\varepsilon$-uniformly stable if for all
%     training datasets $S$, $S'\in \calZ^{nm}$ that only differ in one example,
%     we have:
%     \begin{equation}
%         \sup_{z\in \calZ}\EE_A[\ell(A(S); z) - \ell(A(S');z)] \leq \varepsilon\;.
%     \end{equation}
% \end{definition}

% One can then derive the following renowned lemma,
% linking generalization and uniform stability \citep{bousquet2002stability,
% shalev2010learnability}.

% \begin{lemma}\emph{(Generalization via uniform stability).} \label{lemma:uniform_stab}
% Let A be $\varepsilon$-uniformly stable. Then,
%     \begin{equation}
%         |\EE_{A,S}[R(A(S)) - R_S(A(S))]| \leq \varepsilon\;.
%     \end{equation}
    
% \end{lemma}
% Thanks to this lemma, it suffices to control the uniform stability of $A$ in order to get the desired generalization bound.
%\bat{Batiste : En réalité, on pourrait uniquement considérer la stabilité uniforme. A voir si on supprime le paragraphe suivant.}
Contrary to a large body of works using the well-known \emph{uniform stability} \citep{bousquet2002stability,shalev2010learnability}, our analysis relies on the notion of \emph{on-average model
%One of the most famous and used notion of stability is, arguably, \emph{uniform stability} \citep{bousquet2002stability,shalev2010learnability}. However, in our analysis we rely on the notion of \emph{on-average model
stability} \citep{lei2020fine} which has the advantage to give tighter bounds in our analysis. Below, we recall this notion, with a slight adaptation
to the decentralized setting.

\begin{definition}\emph{(On-average model stability).} \label{def:on-average} Let $S = (S_1,\ldots,S_m)$ with $S_k = \{Z_{1k},\ldots,Z_{nk}\}$ and $\TS = (\TS_1,\ldots,\TS_m)$ with $\TS_k = \{\TZ_{1k},\ldots,\TZ_{nk}\}$ be two independent copies such that $Z_{ik}\sim \calD_k$ and $ \TZ_{ik}\sim \calD_k$. For any $i\in\{1,\ldots, n\}$ and $j\in\{1,\ldots,m\}$, let us denote by $S^{(ij)}=(S_1,\ldots,S_{j-1},S_j^{(i)},S_{j-1},\ldots,S_m)$, with $S_j^{(i)}=\{Z_{1j},\ldots,Z_{i-1j},\TZ_{ij},Z_{i+1j},\ldots,Z_{nj}\}$, the dataset formed from $S$ by replacing the $i$-th element of the $j$-th agent's dataset by $\TZ_{ij}$. A randomized algorithm $A$ is said to be \emph{on-average model $\varepsilon$-stable} if
    \begin{equation}
        \EE_{S,\TS,A}\Big[\frac{1}{mn}\sum_{i=1}^n\sum_{j=1}^m ||A(S) - A(S^{(ij)})||_2\Big] \leq \varepsilon \;.
    \end{equation}

\end{definition}

 A key aspect of on-average model stability is that it can directly be linked to the generalization error, as shown in the following lemma.
 
 \begin{lemma}\emph{(Generalization via on-average model stability \citep{lei2020fine}).} 
 \label{lemma:ob-avg-gen} Let $A$ be on-average model $\varepsilon$-stable.
 Then, if $\ell(\cdot;z)$ is $L$-Lipschitz for all $z\in\calZ$ (see Assumption~\ref{ass:lipschitz}), we have $|\EE_{A,S}[R(A(S)) - R_S(A(S))]|
 \leq L \varepsilon$.
    
 \end{lemma}
 
Thanks to this lemma, it suffices to control the on-average model stability of the decentralized algorithm $A$, in order to get the desired generalization bound.

\subsection{Decentralized SGD}
\label{sec:dsgd}

Throughout this paper, we focus on the popular Decentralized Stochastic
Gradient Descent (D-SGD) algorithm \citep{nedic2009distributed,lian2017can}, which aims to
find minimizers (or saddle points) of the empirical risk $R_S(\theta)$ in a fully
 decentralized fashion. This algorithm is based on peer-to-peer communications
 between agents, where a graph is used to encode which pairs of agents 
(also referred to as nodes) can interact together. %in order to jointly solve the
 % minimization problem.
 More specifically, this \emph{communication graph} is represented by a weight matrix $W\in[0,1]^{m\times m}$, where $W_{jk}>0$ gives the weight that agent $j$ gives to messages received from agent $k$, while $W_{jk}=0$ (no edge) means that $j$ does not receive messages from $k$. %The mixing matrix $W$ is assumed to be doubly stochastic
 
 % \aurelien{typo dans l'algo je crois: $\theta_i^{(t)}$ devrait être $\theta_j^
 %   {(t)}$, et $\sum^n_{k=1}$ devrait être $\sum^m_{k=1}$?}
 % \marc{Effectivement j'ai corrigé}

\begin{algorithm}[h]
    \caption{Decentralized SGD \citep{lian2017can}}\label{alg:d-sgd} 
    \begin{algorithmic}
        \STATE {\bfseries Input:} Initialize $\forall k$, $\theta_k^{(0)} = \theta^{(0)} \in \RR^d$, iterations $T$, stepsizes $\{\eta_t\}_{t=0}^{T-1}$, weight matrix $W$.
        \FOR{$t=0,\ldots,T-1$}
        \FOR{each node $k=1,\ldots,m$ }
            \STATE Sample $I^t_k\sim \calU\{1,\ldots, n\}$ %\Comment{Random sample selection}
            \STATE {\bfseries Variant A:}
            \STATE \quad $\theta_k^{(t+1)} \gets \sum^m_{l=1}W_{kl}\left(\theta_l^{(t)} - \eta_t\nabla \ell(\theta^{(t)}_l;Z_{I^t_ll})\right)$ %\Comment{Variant B}
             \STATE {\bfseries Variant B:}
            \STATE \quad $\theta_k^{(t+1)} \gets \sum^m_{l=1}W_{kl}\theta_l^{(t)} - \eta_t\nabla \ell(\theta^{(t)}_k;Z_{I^t_kk})$ 
        \ENDFOR
        \ENDFOR
    \end{algorithmic}
  \end{algorithm}

\looseness=-1 D-SGD is summarized in Algorithm \ref{alg:d-sgd}. As mentioned in the introduction, there exists two main variants of this algorithm. In Variant A, each agent $k$ first performs a stochastic gradient update based on $\nabla \ell(\theta^{(t)}_k;Z_{I^t_kk})$, i.e., the stochastic gradient
of $\ell$ evaluated at $\theta^{(t)}_k$ with $I^t_k\sim \calU\{1,\ldots, n\}$
the index of the data point uniformly selected by agent $k$ from its local
dataset $S_k$ at iteration $t$. Then, it aggregates its updated
 parameter vector with its neighbors according to the weight matrix $W$. Variant B somehow reverses the two steps: each agent first aggregates its parameter vector with its neighbors and then performs a stochastic gradient update based on the parameter vector it had before the averaging step.

Variant B is often seen as more efficient because the aggregation step and the stochastic gradient calculation can be done in \emph{parallel}, while in Variant A, the aggregation step cannot be performed before all agents have finished their stochastic gradient update. From an optimization perspective however, both variants % are equivalent and
are guaranteed to converge at the same rate~\citep{lian2017can}. 

\looseness=-1 The state of the art currently suggests that these two variants diverge mainly in terms of generalization error. While \citet{richards2020graph} focus on Variant A and show that we can recover the same generalization bounds as centralized SGD, \citet{sun2021stability,zhu2022topology} and \citet{taheri2023generalization} focus on Variant B and show an increased instability due to decentralization. One of our main contributions is to fill this gap in the current theory, by showing that Variant B can also reach the same generalization error as centralized SGD, making it equivalent to Variant A.

\begin{remark}
 \looseness=-1 Due to its fully decentralized nature, D-SGD (both variants) outputs $m$ different parameters $A_1(S) \triangleq \theta^{(T)}_1,\ldots,A_m(S) \triangleq \theta^{(T)}_m$ at the end of the optimization process (one per agent). For this reason, the stability and generalization analysis of the next sections will not be made with respect to a single output $A(S)$ as described in Section \ref{sec:background}, but rather with respect to (one of) these different outputs.
\end{remark}

 \subsection{Main Assumptions}

We focus on the classic setup of \citet{hardt2016train}, also considered by 
\citet{richards2020graph} and \citet{sun2021stability} in prior work on the generalization analysis of
D-SGD. These works rely on the standard assumptions of $L$-Lipschitzness and $\beta$-smoothness of the loss function.

 \begin{assumption}\emph{($L$-Lipschitzness).} \label{ass:lipschitz} We assume that the loss function $\ell$ is differentiable w.r.t. $\theta$ and uniformly Lipschitz, i.e., $\exists L>0$ such that $\forall \theta, \theta' \in \RR^d, z \in \calZ$, $|\ell(\theta;z)-\ell(\theta';z)|\leq L \|\theta - \theta'\|_2$, or equivalently, $\|\nabla \ell(\theta;z)\|_2 \leq L$.

 \end{assumption}

 \begin{assumption}\emph{($\beta$-smoothness).} \label{ass:smooth}The loss function $\ell$ is $\beta$-smooth i.e. $\exists \beta>0$ such that $\forall \theta, \theta' \in \RR^d, z \in \calZ$, $\|\nabla\ell(\theta;z)-\nabla\ell(\theta';z)\|_2\leq \beta \|\theta - \theta'\|_2$.    
 \end{assumption}

\begin{remark}
    By considering Lipschitz and smooth loss functions, our results will be directly comparable to those of \citet{hardt2016train,richards2020graph} and \citet{sun2021stability}. Nevertheless, we expect the conclusions of this paper to be the same for the analyses with relaxed hypotheses \citep{zhu2022topology,taheri2023generalization}. We leave this study for  future research.
    %and keep them for future investigations. 
\end{remark}

Our last assumption
concerns the weight matrix $W$. It is again very standard and used
extensively in the literature of decentralized optimization \citep[see e.g.,]
[]{lian2017can,koloskova2020unified}.

 \begin{assumption}\emph{(Mixing matrix).} \label{ass:doubly}
$W$ is doubly stochastic, i.e., $W^T\1 = W\1 = \1$ where $\1$ is the vector (of
 size $m$) that contains only ones. %and its second-largest eigenvalue (in module) is smaller than one: $|\lambda_2(W)|<1$.
 \end{assumption}

 Note that contrary to what is usually considered in the literature, we do not assume the communication graph $W$ to be connected. As an example, we allow $W$ to be the identity matrix, which would reduce D-SGD to $m$ independent local SGD algorithms.
\section{Generalization Error for Convex Loss Functions}
\label{sec:convex}

This section presents our first main contribution. Focusing on convex and strongly convex functions, we first demonstrate that we can recover the exact same generalization upper bounds for Variant~A \citep[obtained by][]{richards2020graph} and Variant~B of D-SGD. Hence, our bounds unify these two variants and contradict
(and greatly improve upon) the recent results of \citet{sun2021stability,zhu2022topology,taheri2023generalization}. These authors suggested that the generalization error in D-SGD (Variant B) was adversely affected by sparse communication graphs. 
%who claimed that the generalization error of D-SGD (Variant B) was badly impacted by sparse communication graphs. 
Crucially, our analysis demonstrates that the generalization error of
D-SGD, regardless of the variant, is in fact not impacted by the choice of communication graph,
or by decentralization at all, as
%we essentially recover the same bounds proved by \citet{hardt2016train} for \emph{centralized} SGD.
our bounds align closely with those established by \citet{hardt2016train}  for \emph{centralized} SGD.

%Empirically, generalization errors may still reveal differences in the choices of communication graph. Hence, we then propose a more refined generalization bound, referred as "optimization-dependent," which highlights the differences that remain between different communication graphs, under the "worst-case" bounds given in the next section.

% Following prior work 
% \citep{sun2021stability,zhu2022topology,taheri2023generalization}, we prove
% generalization bounds for the \emph{average of final iterates} $\bar{\theta}^{(T)}=\frac{1}{m}\sum_
% {j=1}^m\theta_j^{(T)}$ of D-SGD.\footnote{Considering the average of final
% iterates is also common when analyzing the optimization error of D-SGD, see
% for instance \citep{koloskova2020unified,ying2021exponential}}
% Our bounds contradict
% (and improve upon) the recent results of \citet{sun2021stability}, who claimed
% that the generalization error of D-SGD was badly impacted by sparse
% communication graphs and obtained vacuous bounds for non-connected
%  graphs. Remarkably, our results demonstrate that the generalization error of
%  D-SGD is in fact \emph{not} impacted by the choice of communication graph,
%  or by decentralization at all, as
%  we essentially recover the bounds proved by \citet{hardt2016train} for
%  \emph{centralized} SGD.

%  \bat{REVOIR}

%\subsection{Worst-Case Analysis}

%\label{sec:worst-case}

\subsection{General Convexity}

\begin{theorem} \label{thme:convex}
    Assume that the loss function $\ell(\cdot;z)$ is convex, $L$-lipschitz (Assumption \ref{ass:lipschitz}) and $\beta$-smooth (Assumption \ref{ass:smooth}). Let $A_1(S) = \theta_1^{(T)},\ldots,A_m(S)= \theta_m^{(T)}$ be the $m$ final
    iterates of D-SGD (Variant B) run for $T$ iterations, with mixing matrix $W$ satisfying Assumption~\ref{ass:doubly} and with $\eta_t\leq \frac{2\min_k\{W_{kk}\}}{\beta}$. Then, $\forall k=1,\ldots,m$, $A_k(S)$ has a bounded expected generalization error:
    \begin{equation}
    \label{eq:worst-case-convex}
        |\EE_{A,S}[R(A_k(S)) - R_S(A_k(S))]| \leq \frac{2L^2 \sum_{t=0}^
        {T-1}\eta_t}{mn}\;.
    \end{equation}
\end{theorem}

%\bat{Le sketch peut être enlevé (si besoin)}
\begin{proof}[Sketch of proof (see Appendix \ref{app:convex} for details)] 

    Prior results \citep{sun2021stability,taheri2023generalization} are
    suboptimal because they try to mimic state-of-the-art \emph{optimization
    error} analyses \citep{kong2021consensus} which require to control a 
    \emph{consensus distance} term $\sum_k\|\theta^{(t)}_k - \bar{\theta}^{(t)}\|^2$, where $\bar{\theta}^{(t)} = (1/m)\sum_k{\theta^{(t)}_k}$. This term, important to ensure the minimization of the empirical risk, is small when all local parameters are close to one another, which is the case only if the communication graph is sufficiently connected. 
    %In the following proof, we provide 
    Our proof relies on a tighter analysis, which does not require the control of such consensus distance term.

    Denote by $A_k(S) = \theta_k^{(T)}$ and  $A_k(S^{(ij)}) = \Ttheta_k^{(T)}$, the final iterates of agent $k$ for D-SGD (Variant B) run over two data sets $S$ and $S^{(ij)}$ that differ only in the $i$-th sample of agent $j$ (see Def.~\ref{def:on-average} for notations). The objective is to control the on-average model stability $\frac{1}{mn}\sum_{i,j}\EE[\delta_k^{(T)}(i,j)]$, with $\delta_k^{(T)}(i,j) = \|\theta_k^{(T)} - \Ttheta_k^{(T)}\|_2 $, and then apply Lemma \ref{lemma:ob-avg-gen} to conclude.
    
    The crux of the proof is to recognize, in the updates of Variant B, the gradient updates of a classical SGD with step-size $\eta_t/W_{kk}$, and then use its $1$-expansivity property (Lemma \ref{lem:expansivity} in Appendix \ref{app:lemmas}) when $\eta_t\leq 2\min_k\{W_{kk}\}/\beta$ to obtain the recursion 
    \begin{align}
        \EE[\delta_k^{(t+1)}& (i,j)]  \leq \sum_{l=1}^mW_{kl}\EE[\delta_l^{(t)}(i,j)] \nonumber \\
        & \quad + \frac{2\eta_t}{n}\EE[\|\nabla \ell (\theta_k^{(t)}; Z_{ij})\|_2]\mathds{1}_{\{k=j\}} \nonumber \\
        & \leq \sum_{l=1}^mW_{kl}\EE[\delta_l^{(t)}(i,j)] + \frac{2L\eta_t}{n}\mathds{1}_{\{k=j\}}\;, \label{eq:recursion}
    \end{align}
    where the second inequality is obtained by bounding the gradient norm by $L$ (Assumption \ref{ass:lipschitz}). 
    %At this point of the proof, we are in a scheme similar to that of \citet{richards2020graph} in their proof for Variant~A. 
    Having established Eq.~\eqref{eq:recursion}, the remainder of our proof proceeds along a path similar to that taken by \citet{richards2020graph} in their proof for Variant A.
    We can recursively apply Eq.~\eqref{eq:recursion} until $t=0$ and use the fact that $\delta_k^{(0)}(i,j) = 0$ for all $k$ to get $\EE[\delta_k^{(T)}(i,j)] \leq \frac{2L}{n}\sum_{t=0}^{T-1}(W^{T-t-1})_{kj}\eta_t$. Averaging over $i$ and $j$ and using the fact that any power of $W$ is also doubly stochastic, we obtain that the on-average model stability is upper bounded by $\frac{2L \sum_{t=0}^{T-1}\eta_t}{mn}$, which concludes the proof with a direct application of Lemma \ref{lemma:ob-avg-gen}.
\end{proof}

    % Theorem \ref{thme:convex} 
    % shows that for convex functions, we can recover the exact same (constant factors included) generalization
    % bounds for Variants~A and B of D-SGD \citep[Lemma 13]{richards2020graph}, which are also the same as the one obtained by \citet{hardt2016train} for
    % centralized SGD run over $mn$ data points and with convex loss functions, a result proved to be optimal \citep{zhang2022stability}.
    \citet[Lemma 13]{richards2020graph} and Theorem \ref{thme:convex} demonstrate that, for convex functions, the generalization bounds for both Variants A and B of D-SGD are identical---including constant factors---to those obtained by \citet{hardt2016train} for centralized SGD over $mn$ data points. Moreover, this bound is optimal in the centralized setting   \citep{zhang2022stability}.
    The only difference resides in the fact that, for Variant B, we need to take smaller stepsizes (below $2\min_k\{W_{kk}\}/\beta$ in Theorem \ref{thme:convex} compared to $2/\beta$ for the others). 
    %This difference stems from the greater difficulty in linking the iterates of Variant B to those of a standard gradient descent and is rather mild. %Indeed, in Variant A, each agent directly performs a stochastic gradient descent update before averaging, which preserves the expansivity property of SGD. In the case of Variant B, this is no longer the case, as the gradient is not calculated from the current parameter, which has already been averaged. This specificity calls for a sufficiently small step size, in order to recover the expansivity properties of gradient descent. 
    %This additional hypothesis remains however rather mild. 
    %Indeed, the fact that the step size should be smaller is generally a condition for the optimization of empirical risk. \bat{ADD ref + resultats? @KEVIN: tu avais un truc il me semble?}.
    This difference stems from the greater difficulty in linking the iterates of Variant B to those of a standard gradient descent, but is rather mild, as the assumptions to ensure convergence of the associated optimization problem are usually stronger (see Appendix \ref{app:step_bound} for more details).

    Overall, our theorem strictly improves upon the recent result of the closest work \citep{sun2021stability}, which
    obtained, for Variant B of D-SGD, an upper-bound with an extra additive term: $\frac{2L^2 \sum_{t=0}^
        {T-1}\eta_t}{mn} + \calO(\frac{T}{\rho})$,
    where $\rho \in [0,1]$ is the spectral gap of $W$. 
    % This former result was
    % therefore claiming that the generalization error is strongly impacted by
    % the connectivity of $W$ (tending to infinity as the graph becomes sparser), which is in fact not true as demonstrated by our
    % result. Strikingly, their bound is not even consistent in the
    % sense that it does not tend to $0$ as $n$ grows.
    This earlier result suggested that the generalization error is significantly influenced by the connectivity of~$W$, specifically suggesting that the error would diverge as the graph becomes sparser ($\rho \to 0$). However, our findings contradict this claim.
    Remarkably, the bound by \citet{sun2021stability} does not exhibit consistency, as it fails to approach zero even as $n$ increases.

    % \begin{remark} \label{rmk:avg-worst-case}
    %     In \citet{sun2021stability}, but also \citet{zhu2022topology} and \citet{taheri2023generalization}, the authors do not actually control the generalization error of the $m$ outputs $A_1(S) = \theta_1^{(T)},\ldots,A_m(S)= \theta_m^{(T)}$, but only the one of the average of final iterates $\bar{\theta}^{(T)} = (1/m)\sum_{k=1}^m\theta_k^{(T)}$. Our result was made possible by considering the on-average model stability of the outputs (Def. \ref{def:on-average}), making the average of final iterates useless, and making our result stronger than the previous ones as our bound in Theorem \ref{thme:convex} is also true for $A(S) = \bar{\theta}^{(T)}$ (see Proposition \ref{prop:gen-avg} in Appendix \ref{app:gen-avg}).
    % \end{remark}
    \begin{remark} \label{rmk:avg-worst-case}
        In \citet{sun2021stability}, but also \citet{zhu2022topology} and \citet{taheri2023generalization}, the authors control the generalization error of the averaged final models $\bar{\theta}^{(T)} = (1/m)\sum_{k=1}^m\theta_k^{(T)}$, but not of individual final models. 
        %Our result is stronger for two reasons: (i) Theorem \ref{thme:convex}, and all the results presented in this paper, also apply for the average model $\bar{\theta}^{(T)}$, as shown in Appendix \ref{app:gen-avg}, and (ii) the final averaged parameter is not really accessible in the fully decentralized setting, as such averaging does not represent an actual step of D-SGD. Hence, looking at the generalization properties of this parameter does not really make sense here. 
        In this sense, our result is stronger, for two key reasons. Firstly, Theorem \ref{thme:convex} and all  results in this paper can be directly extended to the average model $\bar{\theta}^{(T)}$, as detailed in Appendix \ref{app:gen-avg}. Secondly, this average model is not computed at any stage of the D-SGD process (except  where the communication graph is complete). Therefore, examining the generalization properties of this parameter introduces a certain conflict with the fully decentralized context.
    \end{remark}

    %The fact that the generalization bound for D-SGD perfectly matches the one of centralized SGD may seem surprising at first. This come from the fact that we are considering a worst-case analysis which somehow hides the true impact of the graph for certain type of loss functions and data-distributions. In Section~\ref{sec:data-dep}, we will investigate this question, providing a new analysis that can exhibit the impact of the graph in some regime.
    %The fact that the generalization bound for D-SGD perfectly matches that of centralized SGD might initially appear surprising. This arises from our consideration of a worst-case analysis, which tends to obscure the true influence of the graph on certain types of loss functions and data distributions. In Section~\ref{sec:data-dep}, we delve into this issue, offering a new analysis that highlights the graph's impact in some regimes.

    \subsection{Strong Convexity}

We now consider strongly convex functions. As such functions cannot be
Lipschitz (Assumption \ref{ass:lipschitz}) over $\IR^d$, we restrict our
analysis to the optimization over a convex compact set $\Theta$ as done by
\citet{hardt2016train}. Denoting by $\Pi_\Theta(\Ttheta) = \argmin_{\theta \in \Theta}\|
\Ttheta - \theta\|$ the Euclidean projection onto $\Theta$, we consider the
\emph{projected} extension of the D-SGD algorithm, which replaces the
updates from Algorithm~\ref{alg:d-sgd} by: 
\begin{equation*}
    \theta_j^{(t+1)} \hspace{-0.1cm} \gets \hspace{-0.1cm} \left\{
        \begin{array}{ll}
                    \hspace{-0.2cm} \sum^m_{k=1}W_{jk}\Pi_\Theta\Big(\theta_k^{(t)} - \eta_t\nabla \ell(\theta^{(t)}_k;Z_{I^t_kk})\Big) \hspace{0.2cm} \text{(A)}\\[0.3cm]

            \hspace{-0.2cm} \Pi_\Theta\Big(\sum^m_{k=1}W_{jk}\theta_k^{(t)} - \eta_t\nabla \ell(\theta^{(t)}_j;Z_{I^t_jj})\Big) \hspace{0.2cm} \text{(B)}
        \end{array}
      \right.
\end{equation*}

% $\theta_j^{(t+1)} \gets \sum^m_{k=1}W_{jk}\theta_k^{(t)} - \eta_t\nabla \ell(\theta^{(t)}_j;Z_{I^t_jj})$ \Comment{Variant B}
%             \State $\theta_j^{(t+1)} \gets \sum^m_{k=1}W_{jk}\left(\theta_k^{(t)} - \eta_t\nabla \ell(\theta^{(t)}_k;Z_{I^t_kk})\right)$
%Before moving to our generalization result, we note that this algorithm is
%well-suited to solving Tikhonov regularization problems, which
%makes it quite natural to consider in practice.
%We observe that by imposing a bound on the norm of the model, this algorithm is implicitly applying a form of regularization.

\begin{theorem} \label{thme:strongly} Assume that the loss function $\ell
(\cdot;z)$ is $\mu$-strongly convex, $L$-Lipschitz over $\Theta$
(Assumption~\ref{ass:lipschitz}) and $\beta$-smooth (Assumption~\ref{ass:smooth}). Let $A_1(S) = \theta_1^{(T)},\ldots,A_m(S)= \theta_m^{(T)}$ be the $m$ final iterates of the projected D-SGD (Variant B) run for $T$ iterations, with mixing matrix $W$ satisfying Assumption~\ref{ass:doubly} and with constant stepsize $\eta\leq \min_k\{W_{kk}\}/\beta$. Then, $\forall k=1,\ldots,m$, $A_k(S)$ has a bounded expected generalization error:
    \begin{equation}
        |\EE_{A,S}[R(A_k(S)) - R_S(A_k(S))]| \leq \frac{4L^2}{\mu mn}\;.
    \end{equation} 
    
\end{theorem}

    The proof of Theorem \ref{thme:strongly}, provided in Appendix 
    \ref{app:strongly-convex}, essentially follows the same scheme as the one
    derived above for convex functions. Once again, the bound matches the
    optimal one obtained for centralized SGD with strongly convex functions in
    \citet{hardt2016train}, and the one obtained by \citet{richards2020graph} for Variant A of D-SGD.
    Notice that, contrary to the general convex case, the generalization bound for strongly convex functions is independent of the number of iterations $T$, which makes these problems more stable and less likely to overfit.

    %In the work of 
    \citet{sun2021stability} derive a similar generalization bound but with an extra additive
    error term in $\calO(\frac{1}{\mu\rho})$. Their 
    bound is therefore strictly weaker: it is \emph{not} converging to $0$ as the number of samples
    increases and is vacuous when the communication graph is not connected ($\rho=0$). This again illustrates the suboptimality of these previous results and the major improvement brought by ours.

\subsection{Deriving excess risk bounds}

\label{sec:excess}

Recall from Section \ref{sec:background} that the main objective of statistical learning is to control the excess risk $\epsilon_{\text{excess}}\triangleq \mathbb{E}_{A,S}  [R(A(S)) - R(\theta^\star)]$, which can be upper-bounded by the sum of the generalization error ($\epsilon_{\text{gen}}$) and the optimization error ($\epsilon_{\text{opt}}$).

% , i.e.:
% $$\epsilon_{\text{excess}}\leq \epsilon_{\text{gen}} + \epsilon_{\text{opt}},$$
% where  $\epsilon_{\textrm{gen}} \triangleq \mathbb{E}_{A,S}[R(A(S)) - R_S(A(S))]$ and $\epsilon_{\textrm{opt}} \triangleq \mathbb{E}_{A,S}[R_S(A(S)) - R_S(\widehat{\theta}_{\text{ERM}})]$.

 Our work is centered on the control of the generalization error. However, with the rather abundant literature on the control of optimization errors for D-SGD \citep{koloskova2020unified,neglia2020decentralized,ying2021exponential}, one can combine the results of these papers with ours to obtain bounds on the excess risk, as explained below. Note that most bounds on the optimization error from the literature are given for the averaged parameter $A(S)=\bar{\theta}^{(T)}$ or $A(S) = \frac{1}{T}\sum_{t=1}^T\bar{\theta}^{(t)}$. Since our generalization bounds are also valid for these averaged parameters (see Section \ref{app:gen-avg}, proofs for the time-average parameter are analogous), the following discussions are made with respect to them.

For \textbf{convex functions}, one can adapt the optimization error bound from \citet[Proposition 3.1]{neglia2020decentralized} to our notations and obtain $\epsilon_{\text{opt}}=\mathcal{O}\Big(\frac{1}{\eta T}+ \frac{\eta L^2}{\rho}\Big)$, where some constant factors have been omitted for simplicity. Combining this with our generalization bound of order $\mathcal{O}(\frac{T\eta L^2}{mn})$, one can take $T = \Theta(\frac{\sqrt{mn}}{\eta L})$ with $\eta \leq \min\{\frac{\rho}{L\sqrt{mn}};\frac{2\min_k{W_{kk}}}{\beta}\}$, and recover the classical rate of order $\mathcal{O}(\frac{L}{\sqrt{mn}})$ for $\epsilon_{\text{excess}}$, a rate that can be found for instance in \citet{hardt2016train} or \citet{lei2020fine} for centralized SGD.

For \textbf{$\mu$-strongly convex functions}, our result from Theorem \ref{thme:strongly} exhibits a generalization bound independent of the algorithm parameters $\eta$, $W$ and $T$ (as soon as $\eta$ satisfies the constraint of our theorem). Moreover, we know \citep[see e.g.][]{koloskova2020unified,neglia2020decentralized} that the optimization error $\epsilon_{\text{opt}}$ can be set arbitrary small (in particular smaller than $\epsilon_{\text{gen}}$), as soon as the graph $W$ is connected, the number of iterations $T$ is sufficiently large and the stepsize $\eta$ is sufficiently small (in particular satisfying our constraint). Hence, as soon as $W$ is connected, there exists an instance of parameters $\eta$, and $T$ of D-SGD that gives an excess risk bound $\epsilon_{\text{excess}}$ with a "fast" rate of order $\epsilon_{\text{gen}}$ = $\mathcal{O}(\frac{L^2}{\mu mn})$, a rate that can be found for instance in \cite{hardt2016train} for centralized SGD. %\bat{Finally, note that the optimization error bound from \citet{koloskova2020unified} is given for the Variant A of D-SGD. However, as noted by \citet{lian2017can}, the optimization error analysis is similar between the two variants, so that it does not change the conclusion.}

We emphasize that the convergence of the above excess risk bounds actually depend on the communication graph. Indeed, the convergence is possible only if the graph is connected (i.e. $\rho<1$, a necessary condition to control the optimization error), and the number of iterations $T$ needed to make the optimization error small depends on $\rho$.

\section{Generalization Error for Non-Convex Loss Functions}
\label{sec:non-convex}

%The less-studied 
The case of non-convex (but bounded) loss functions was only investigated by \citet{sun2021stability}, for Variant B of D-SGD.
%with bounded non-convex functions. 
%As in convex cases, they derive a generalization upper bound similar to that of \citet{hardt2016train}, but with an additional term not converging with the increasing sample size.
Similar to their findings in convex scenarios, they established a generalization upper bound 
akin to that of \citet{hardt2016train}. However, their bound again includes an additional term that does not diminish with increasing sample size.
This raises the following question: 
%As in Section \ref{sec:convex}, 
can a finer-grained analysis
than that of \citet{sun2021stability} recover, \emph{for both variants of D-SGD}, a result analogous to that of \citet[Theorem~3.12]{hardt2016train}
%, obtained 
for centralized SGD with bounded non-convex loss functions?

To answer this question, we adopt the set of hypotheses of \citet{hardt2016train} and seek to extend their proof technique to the decentralized framework. Hereafter, we provide our generalization  bound for bounded non-convex functions. %The proof is deferred to Appendix \ref{app:non-convex}.  %Their proof relies on two key points that differ from convex analyses. Firstly, they use a uniform stability argument (Lemma \ref{lemma:uniform_stab}) rather than an on-average model stability argument (Lemma \ref{lemma:ob-avg-gen}). Then, by conditioning their analysis on the time $t_0$ at which the swapped sample is first selected, they can minimize their upper bound with respect to $t_0$ and thus obtain a tighter bound.

%These differences are not without consequences for our analysis. Indeed, one of the main reason why the graph $W$ had no impact on the worst-case analysis with convex losses was the use of the on-average model stability argument. As a matter of fact, a close look at Equation \eqref{eq:recursion} and discussion below reveal that the model stability depends on the communication graph, the latter then disappear after averaging with respect to $i$ and $j$. Without such an averaging step, we therefore cannot hope to obtain a generalization bound independent of the graph. %For this reason, we split our worst-case analysis in two parts. First, we analyse the average of generalization errors, where the averaging step will allow to erase the impact of the communication graph and recover a result analogue to that of \citet{hardt2016train}. Then, we perform a node-wise generalization analysis which, contrary to Theorem \ref{thme:convex} and \ref{thme:strongly} will this time reveal some impact of $W$.

%We first focus on the average of final iterates i.e. $A(S)=
% \bar{\theta}^{(T)}$.

\begin{theorem}
    \label{thme:non-convex}
    
Assume that $\ell(\cdot;z) \in [0,1]$ is an $L$-Lipschitz
    (Assumption~\ref{ass:lipschitz}) and $\beta$-smooth
    (Assumption~\ref{ass:smooth}) loss function for every $z$. Let $A_1(S),\ldots,A_m(S)$ be the $m$ final
    iterates of D-SGD (Variant A and B) run for $T$ iterations, with mixing matrix~$W$ satisfying Assumption~\ref{ass:doubly}, such that $\min_k\{W_{kk}\}>0$, and with monotonically non-increasing step sizes $\eta_t\leq \frac{c}{t+1}$, $c>0$. Then, we have:
    \begin{align}
        |\EE_{A,S} [R(A_k(S)) & - R_S(A_k(S))]| \nonumber \\
        & \leq (1+\frac{1}{\beta c})(2cL^2)^{
        \frac{1}{\beta c +1}}\frac{T^{\frac{\beta c}{\beta c +1}}}{nm^{\frac{1}{\beta c +1}}}\;.
    \end{align}    
\end{theorem}

\begin{proof}[Sketch of proof (see Appendix \ref{app:non-convex} for details)] The crux of the proof is to condition the analysis on the time $t_0$ at which the swapped sample is first selected, then it is possible to minimize the generalization upper bound with respect to $t_0$ and obtain a tighter bound.
%However while \citet{hardt2016train} use this technique combined with uniform stability arguments, we seek to use it with an on-average model stability argument instead. This is crucial as this is what makes the impact of the graph disappear. 
\citet{hardt2016train} employed this same technique in conjunction with uniform stability arguments; our approach diverges by integrating it with an on-average model stability argument. This distinction is critical, as it is precisely what eliminates the impact of the graph in our analysis.
%Indeed, a close look at Equation~\eqref{eq:recursion} and the discussion below reveal that, before averaging it, the model stability was depending on the communication graph. 
Indeed, the discussion after Equation~\eqref{eq:recursion} shows that, before averaging over $i$ and $j$, model deviations $\delta_k(i,j)$ depend on the powers of the communication graph $W$.
The appropriate conditioning on $t_0$ combined with on-average model stability is given in the key Lemma \ref{lemma:key-non-conv} and ensures that $\forall t_0\in\mathbb{N}$ we have
\begin{equation*}
%\label{eq:right-hand}
    |\EE_{A,S}[R(A_k(S)) - R_S(A_k(S))]| \leq \frac{t_0}{n} + L\Delta_k^{(T)}\;,
\end{equation*}
where $\Delta_k^{(t)} = \frac{1}{mn}\sum_{i,j}\EE[\delta_k^{(t)}(i,j)\big| {\delta}^{(t_0)}(i,j) = \mathbf{0}]$ and ${\delta}$ is the vector containings $\delta_1,\ldots,\delta_m$ (see the proof sketch of Theorem \ref{thme:convex} for other notations). Then, using the $(1+\eta_t\beta)$-expansivity property of our updates, combined with the fact that $\Delta_k^{(t_0)} = 0$, we prove that:
\begin{equation*}
    \Delta_k^{(T)} \leq \frac{2L}{\beta mn}\Big(\frac{T}{t_0}\Big)^{c \beta}\;.
\end{equation*}
It then suffices to plug this equation into the first one and to minimize it with respect to $t_0$ to complete the proof.
\end{proof}

For the clarity of the discussion, below we omit constant factors in $\beta$, $c$ and $L$, but we stress that our bound has the exact same constant factors as those in \citet{hardt2016train}. Our generalization bound is of order $\calO(T^{\frac{\beta c}{\beta c
+1}}/nm^{\frac{1}{\beta c +1}})$ and several comments can be made. First,
contrary to the convex cases, our bound does not exactly match the one of 
\citet{hardt2016train}. Indeed, when centralized SGD is run over $mn$ data
points, they obtain a bound of order $\calO(T^{\frac{\beta
c}{\beta c +1}}/nm)$ which is strictly better than our bound. This comes from
the fact that the proof technique relies on characterizing the number of steps
that occur before the algorithm picks the data point that differs in $S$ and
$S^{(ij)}$. In centralized SGD, the probability to pick this
point is $1/mn$ at each iteration, while it is only $1/n$ for D-SGD.
Importantly, this means that the weaker bound is not directly due to
decentralization, but rather to the fact that D-SGD selects $m$ samples at
each iteration (instead of only one for SGD). A fairer comparison% would
thus be to compare D-SGD to centralized SGD with batch size $m$.

Importantly, our generalization bound is valid for the two variants of D-SGD, is still independent of the
choice of communication graph, and tends towards $0$ as $n$ and $m$ increase. This
significantly improves the results obtained for
D-SGD in the prior work of \citet{sun2021stability} for Variant B only, where the obtained bound has an extra additive term of order $\calO(T^{\frac{\beta c}{\beta c +1}}C_\rho)$, where
$C_\rho$ depends on the spectral gap $\rho$ of $W$ and can be arbitrarily large.  Note that, as in convex cases (Remark \ref{rmk:avg-worst-case}), their result is given for the average of final iterates, for which our result is also valid (see Proposition~\ref{prop:gen-avg-nonconv} in Appendix~\ref{app:gen-avg}).  

\begin{remark}
    As pointed out in Table \ref{tab:a}, there is a mistake in the original proof of the upper bound of  \citet[Theorem 3]{sun2021stability}. We explain this in Appendix \ref{app:mistake_sun}.
\end{remark}

Finally, notice that for non-convex functions, the excess population risk $\epsilon_{\text{excess}}$ cannot be directly upper-bounded and a discussion analogue to the one of Section \ref{sec:excess} cannot be made. Indeed, in the non-convex case, optimization errors usually control an upper bound on the gradient norm of the objective function and not on function values, which is required by our definition of $\epsilon_{\text{opt}}$.

    \section{Towards Optimization-Dependent Generalization Bounds}
    \label{sec:data-dep}

    Based on the results of the previous sections, one could conclude that decentralization and the choice of communication graph do not have an impact on the generalization of D-SGD. %However in practice, different communication graphs can lead to different behaviors for the associated generalization errors. 
    In this section, we suggest that this rather counter-intuitive result comes from the fact that the previous analyses are ``worst-case'', thereby hiding the true influence of the graph on certain types of loss functions and data distributions. %Although optimal, the previous generalization bounds were indeed extensively relying on the $L$-Lipschitzness property of the loss to bound gradient norms by $L$. While in some regimes this is optimal, there exist other regimes where we can hope to do better. 
    In order to highlight such potential effects, we propose to investigate a certain type of generalization bounds referred to as \emph{optimization-dependent}. This type of refined bounds, also referred as ``data-dependent'' in \citet{kuzborskij2018data}, has been widely investigated in the generalization error analysis of \emph{centralized} gradient methods \citep{kuzborskij2018data,lei2020fine}. Also based on algorithmic stability arguments, they reveal that a good optimization of the empirical risk can be beneficial for generalization (hence the term ``optimization-dependent''). Since it is well known that the choice of graph affects the specific trajectories of the optimization algorithm and have an impact on the optimization error \citep{koloskova2020unified}, we can expect this type of analysis to be appropriate to reveal the impact of the graph's connectivity on generalization. 
    
    Let us start with an additional assumption.

    %\hspace{-1cm}
    
    \begin{assumption}\emph{(Bounded empirical variance).} \label{hyp:bounded-variance} For all agents $k=1,\ldots, m$, training dataset $S\in \calZ^{mn}$ and model parameter $\theta\in \IR^d$, there exists $\sigma^2>0$ such that $\frac{1}{n}\sum_{i=1}^n\|\nabla\ell(\theta;Z_{ik}) - \nabla R_{S_k}(\theta)\|_2^2 \leq \sigma^2$.
    \end{assumption}
    
    This assumption is rather standard in the control of the optimization error of stochastic gradient methods. Here, it is necessary to reveal different regimes in our new optimization-dependent bound, but notice that Assumption \ref{hyp:bounded-variance} is always satisfied under Assumption \ref{ass:lipschitz}, with $\sigma^2=L^2$.
    
    The following lemma first links generalization errors to empirical risk minimization errors. Here, we focus on averaged generalization errors, instead of looking at the one of any fixed agent. This makes the analysis more
    tractable and reveals interesting links with optimization
    errors.

    \begin{lemma}\emph{(Link with local optimization errors).} \label{lemma:link-local-optim} Under the same hypotheses as in  Theorem \ref{thme:convex} and additional Assumption \ref{hyp:bounded-variance}, we have:
        \begin{align*}
            &\frac{1}{m}\sum_{k=1}^m|\EE_{A,S}[R(A_k(S)) - R_S(A_k(S))]| \\
            & \quad \leq \frac{2L\sigma}{mn}\sum_{t=0}^{T-1}\eta_t+ \frac{2L}{mn}\sum_{t=0}^{T-1}\eta_t\frac{1}{m}\sum_{j=1}^{m}\EE[\|\nabla R_{S_j}(\theta_j^{(t)})\|_2]
        \end{align*}
    \end{lemma}

    Lemma \ref{lemma:link-local-optim}, and all the results of this section, are proved in Appendix \ref{app:data-dep}. 
    
    From Lemma \ref{lemma:link-local-optim}, we notice that we can bound $\sigma$ and the gradient norms $\|\nabla R_{S_j}(\theta_j^{(t)})\|_2$ by $L$ and recover (up to a factor~$2$), the worst-case upper bound from Theorem~\ref{thme:convex}. This illustrates well the ``worst-case'' notion mentioned above, but also the fact that the upper bound of Lemma \ref{lemma:link-local-optim} can be better than the one of Theorem \ref{thme:convex} in some regimes. This will notably be the case in ``low noise'' regimes, when $\sigma \ll L$, and when the expected gradient norms $\EE[\|\nabla R_{S_j}(\theta_j^{(t)})\|_2]$ reach small values. Interestingly, these gradient norms are linked with the optimization error of local empirical risks: $\EE[\|\nabla R_{S_j}(\theta_j^{(t)})\|_2]$ will get smaller as the parameter $\theta_j^{(t)}$ of the $j$-th agent minimizes the associated local empirical risk $R_{S_j}(\theta)$. %, and conversely in convex cases.
    In other words, the more rapidly each agent optimizes its own local empirical risk, the smaller (and the better) the bound in Lemma \ref{lemma:link-local-optim}.

    % \begin{remark} \label{rmk:avg-vs-nodewise}
    %     In Lemma \ref{lemma:link-local-optim} and subsequent results, we are looking at a different generalization error that the one from Section \ref{sec:worst-case}. Here, we are looking at the averaged generalization error, whereas we were looking at the one of any fixed agent before. This difference makes the analysis more tractable and still reveals interesting links with optimization errors. However, we must keep in mind that conclusions for a node-wise generalization error might be different and are currently kept for future work. \bat{Donner version du Lemme pour le cas node-wise en appendix ? Lemma \ref{lemma:data-dep-nodewise} }
    % \end{remark}

% \begin{remark}\label{rmk:avg-data-dep}
%     Similarly to Remark \ref{rmk:avg-worst-case}, all the results of this section are also valid for the generalization error of the average of final iterates $A(S)=\bar{\theta}^{(T)}$ considered in related works (Appendix \ref{app:gen-avg}). We chose to present the result on the average of generalization errors over the one on the averaged parameter $\bar{\theta}^{(T)}$ as the latter is not really accessible in the fully decentralized setting, making the result less significant.
% \end{remark}

The fact that the agents should minimize their \emph{local} empirical risks may seem surprising at first. Indeed, as opposed to the minimization of the \emph{full} empirical risk, it suggests that local SGD (D-SGD with identity graph) should be preferred to D-SGD with any other communication graph. This quite counter-intuitive observation comes from the averaging of  the node-wise generalization errors, which increases stability. % and leads to such conclusions.
Using tools from convex optimization, we can now provide a more explicit upper bound, given in the following theorem. % highlighting the idea that the graph should be close to the identity.

\begin{figure*}[!h]
	\centering
	\includegraphics[width=0.4\linewidth]{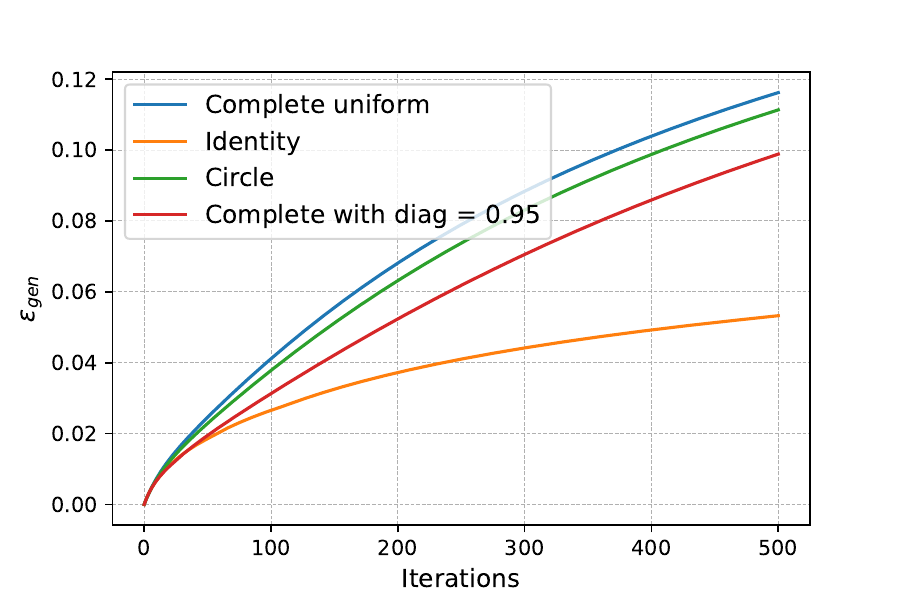}
	\includegraphics[width=0.4\linewidth]{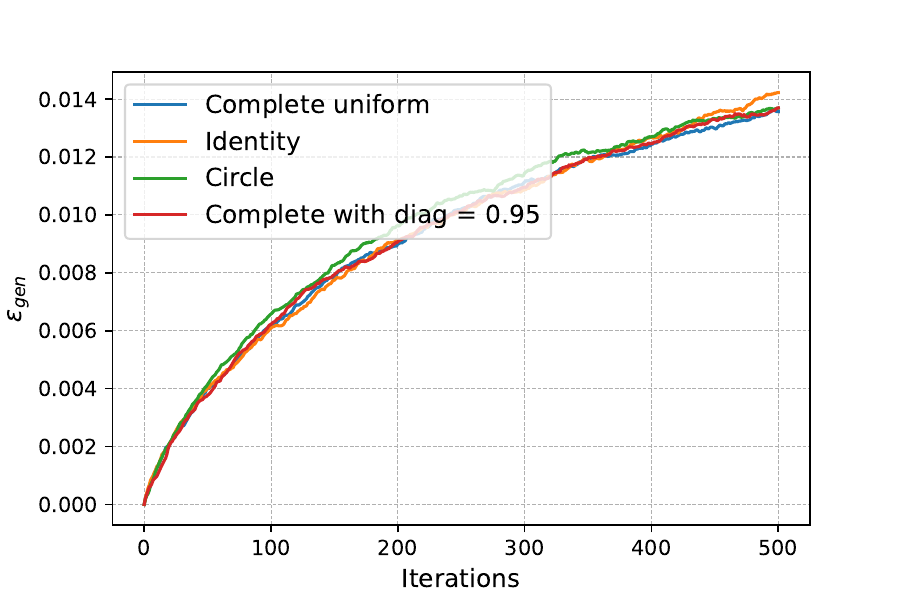}
	\caption{Empirical generalization error, as a function of the number of iterations $T$, and for different communication graphs. Constant stepsize $\eta=0.03$. (Left) Low-noise regime with $\sigma\simeq 0$. (Right) Noisy regime with $\sigma > 0$. See Appendix \ref{app:exps} for experimental details.}
	\label{fig:convex-avg}
\end{figure*}

    \begin{theorem}\emph{(optimization-dependent generalization bound)} \label{thme:optim-dep} Consider the same setting as in Theorem \ref{thme:convex}, with a constant stepsize $\eta$ and additional Assumption \ref{hyp:bounded-variance}. Assume further that $W$ is symmetric. Then, there exists a graph-dependent constant $C_W<\infty$ such that:
\begin{align*}
    \frac{1}{m}\sum_{k=1}^m &|\EE_{A,S}[R(A_k(S)) - R_S(A_k(S))]| \\
    & \leq \frac{2\sqrt{2}L\sqrt{T\eta}}{mn}\sqrt{\frac{1}{m}\sum_{j=1}^{m}\EE[R_{S_j}(\theta^{(0)}) - R_{S_j}(\theta_{S_j}^*)]} \\
    & \hspace{0.8cm} + \frac{2L\sigma\eta T}{mn} + \frac{2L\sqrt{\beta}\sigma\eta^{\frac{3}{2}} T}{mn} + \frac{2L^2T\eta C_W}{mn}\;,
\end{align*}
where $\theta_{S_j}^*$ is the (local) empirical risk minimizer of $R_{S_j}$.
\end{theorem}
Here, $C_W$ corresponds to an upper bound on the series $C_W^{(t)} \triangleq \sum_{s=0}^{t-1}\|W^s - W^{s+1}\|_2$. Its existence is guaranteed (see Lemma \ref{lemma:C_W} in Appendix \ref{app:data-dep}), but unfortunately in most cases  $C_W^{(t)}$ and $C_W$ do not have a closed form expression. However, it can be shown for instance that $C_W=0$ for $W=I$ (local SGD) or $C_W=1$ for $W=\frac{1}{m}\1\1^T$ (complete graph with uniform weights). More generally, a condition for $C_W$ to be small is to be close to the identity graph.

\textbf{Upper bound analysis.} The first term of the upper bound is of order $\calO(\frac{\sqrt{T}}{mn})$ and depends on the averaged optimization error at initialization. This illustrates that if we are good at the initial point $\theta^{0}$, few steps of D-SGD are going to be necessary before reaching a stable point. The other terms are of order $\calO(\frac{T}{mn})$, the last one being graph-dependent with the constant $C_W$, while the other two  depend on the variance $\sigma^2$. From the worst case point of view, we therefore recover the rate $\calO(\frac{T}{mn})$ provided in Theorem \ref{thme:convex}. However, this new bound is more informative as it showcases other regimes. For instance, when $\sigma$ and $C_W$ are sufficiently small, the first term becomes dominant, and the bound becomes of order $\calO(\frac{\sqrt{T}}{mn})$, which strictly improves the worst-case upper bound. Last but not least, if $W=I$ and $\eta\leq\frac{1}{\sqrt{T}}$, we have $C_W=0$ and we obtain a bound of order  \looseness = -1
\begin{equation*}
    \calO\left(\frac{\sqrt[4]{T}}{mn}\sqrt{\frac{1}{m}\sum_{j=1}^{m}\EE[R_{S_j}(\theta^{(0)}) - R_{S_j}(\theta_{S_j}^*)]} + \frac{\sigma \sqrt{T}}{mn} \right)\;,
\end{equation*}

which is the same as the optimization-dependent bound obtained in \citet[Theorem 3]{kuzborskij2018data} for centralized SGD with $mn$ data points. This illustrates that our results generalize those obtained for centralized SGD in past studies.

In Figure \ref{fig:convex-avg}, we represent the generalization errors observed empirically for different communication graphs (see Appendix \ref{app:exps} for experimental details). In the low noise regime (left plot), we observe that the generalization error is strongly impacted by the choice of graph, the best one being the identity. This is in line with our analysis, as in the low noise regime the graph-dependent term of our bound becomes dominant. On the contrary, in the high noise regime (right plot), the choice of communication graph is less significant as we essentially recover the worst-case behavior, in which the choice of communication graph does not matter. Interestingly, we observe that in the firsts iterations of the low-noise regime, all curves have the same slope. This suggests that during this phase, generalization errors evolve linearly and do not depend on the graph, exhibiting the worst-case behavior described by Eq. \eqref{eq:worst-case-convex}. Then, as the algorithm continues, the optimization progresses (depending on the graph), making the algorithm more stable and the worst-case bound too conservative.

In the end, after refuting results which claimed that a poorly connected graph was detrimental to generalization, our bound and our empirical results show that such a graph can, on the contrary, help generalization in certain regimes. This can be contrasted, however, with the fact that the optimization error $\epsilon_{\textrm{opt}}$ of the \emph{full} empirical risk must also be controlled, which can only be done with a connected graph.
Overall, our analysis paves the way for the future development of optimization-dependent generalization bounds, whose ability to characterize the practical impact of decentralization and choice of graph is well illustrated by our results.

\section{Conclusion}

In this paper, we 
showed that previous generalization error analyses of Variant~B of D-SGD were very loose and led to incorrect conclusions
regarding the impact of decentralization on generalization. On the contrary, we show that Variants A and B recover upper bounds analogous to those obtained in the centralized setting, suggesting that decentralization and the choice of graph do not have an impact on generalization. We then argue that this result is coming from a worst-case analysis and propose a refined bound revealing that the choice of graph can in fact improve the worst-case bound in certain regimes, and that a poorly-connected graph can even be beneficial for generalization.

All our generalization results, however, should not be completely dissociated from the
optimization error. As seen in Section \ref{sec:excess}, if we want to recover the optimal excess risk bounds from the centralized setting, the optimization error must be sufficiently small. Contrary to the generalization error, this means that the graph should be connected and the number of iterations sufficiently large (depending on the connectivity of the communication graph). Future work could therefore include a better understanding of the generalization-optimization trade-off, notably with respect to the minimum number of iterations needed to reach the optimal bounds. In an other vein, future investigations could relax the Lipschitz condition, by considering an analysis similar to those proposed by \citet{lei2020fine,schliserman2022stability,schliserman2023tight}, or develop more refined optimization-dependent generalization bounds that would be able to capture, for instance, the impact of data heterogeneity between agents.

% \textbf{although generalization is
% mildly affected by the choice of the graph, the optimization error remains
% heavily impacted by it}.

% \begin{remark}
%     Although standard, Assumptions \ref{ass:lipschitz} and \ref{ass:smooth} are not strictly necessary.
%     %could be relaxed. 
%     The $L$-Lipschitz condition can be avoided by considering an analysis similar to the one proposed by \citet{lei2020fine}, and more generally they can both be relaxed by imposing instead a self-bounding property on the gradients \citep{schliserman2022stability,schliserman2023tight}.
% \end{remark}

% \bat{Future investigations could include the relaxation of certain assumptions, or the development of more refined optimization-dependent generalization bounds that would be able to capture, for instance, the impact of data heterogeneity between agents.}

% Meilleur analyses + resultats coherent par rapport à ce qui est attendu\\
% Technique de preuve simple pouvant être adapté à d'autres techniques de preuve centralisées. \\
% Prendre neanmoins les resultats avec des pincettes : ne pas oublier que l'erreur d'optim est importante et que dans ces cas là le graphe importe bcp \\

% Future: lower bound + relaxed hypothesis + data-dependent gen bound + high probability ?  

\section*{Acknowledgments}

This work was supported by the French government managed by the Agence Nationale de la Recherche (ANR) through grant ANR-20-CE23-0015 (Project PRIDE), through the 3IA Côte d’Azur Investments in the Future project with the reference number ANR-19-P3IA-0002 and under France 2030 program with the reference ANR-23-PEIA-005 (REDEEM project).
It was also funded in part by the European Network of Excellence dAIEDGE under Grant Agreement Nr. 101120726 and by the Groupe La Poste, sponsor of the Inria Foundation, in the framework of the FedMalin Inria Challenge. Batiste Le Bars was supported by an Inria-EPFL fellowship.

\section*{Impact Statement}
This paper presents work whose goal is to advance the field of Machine Learning. There are many potential societal consequences of our work, none which we feel must be specifically highlighted here.


\bibliography{biblio}
\bibliographystyle{icml2024}

%%%%%%%%%%%%%%%%%%%%%%%%%%%%%%%%%%%%%%%%%%%%%%%%%%%%%%%%%%%%%%%%%%%%%%%%%%%%%%%
%%%%%%%%%%%%%%%%%%%%%%%%%%%%%%%%%%%%%%%%%%%%%%%%%%%%%%%%%%%%%%%%%%%%%%%%%%%%%%%
% APPENDIX
%%%%%%%%%%%%%%%%%%%%%%%%%%%%%%%%%%%%%%%%%%%%%%%%%%%%%%%%%%%%%%%%%%%%%%%%%%%%%%%
%%%%%%%%%%%%%%%%%%%%%%%%%%%%%%%%%%%%%%%%%%%%%%%%%%%%%%%%%%%%%%%%%%%%%%%%%%%%%%%
\newpage
\appendix
\onecolumn

% !TEX root = ../neurips_2023.tex

\begin{center}
    {\Large\textbf{Appendix}}
\end{center}
\section{Technical lemmas}
\label{app:lemmas}

Below, we provide important definitions and lemmas that are going to be useful in our analysis. All proofs can be found in \citet{hardt2016train}.

%Let $G_{\eta,z}(\theta) = \theta - \eta \nabla \ell(\theta;z)$ be the (stochastic) gradient update rule with $\eta>0$ and $z\in\calZ$.

\begin{definition}
    \label{def:sgd-update-rule}
    The (stochastic) gradient update rule with $\eta>0$, $z\in\calZ$ and loss function $\ell$ is given by $$G_{\eta,z}(\theta) = \theta - \eta \nabla \ell(\theta;z).$$
\end{definition}

\begin{definition}
    An update rule $G(\theta)$ is said to be $\nu$-expansive if:
    \begin{equation*}
        \sup_{\theta,\theta'}\frac{\|G(\theta) - G(\theta')\|_2}{\|\theta - \theta'\|_2} \leq \nu .
    \end{equation*}
\end{definition}

\begin{lemma}\emph{(Expansivity of $G_{\eta,z}$).} \label{lem:expansivity} If $\ell$ is $\beta$-smooth (Assumption \ref{ass:smooth}), we have:

    \begin{enumerate}
        \item $G_{\eta,z}(\theta)$ is $(1+\eta \beta)$-expansive;
        \item Assume in addition that $\ell(\cdot;z)$ is convex and $\eta<2/\beta$. Then $G_{\eta,z}(\theta)$ is $1$-expansive;
        \item Assume in addition that $\ell(\cdot;z)$ is $\mu$-strongly convex and $\eta<\frac{2}{\beta+\mu}$. Then $G_{\eta,z}(\theta)$ is $(1-\frac{\eta\beta\mu}{\beta+\mu})$-expansive.
    \end{enumerate}
    
\end{lemma}

\begin{lemma}\emph{(Growth recursion)} \label{lem:growth} Fix an arbitrary sequence of gradient update rule $G_{\eta_1,z_1}, \ldots, G_{\eta_T,z_T}$ and another sequence $G_{\eta_1,z'_1}, \ldots, G_{\eta_T,z'_T}$ with same loss function $\ell$ (Def. \ref{def:sgd-update-rule}). Let $\theta_0 = \theta'_0$ be a starting point in $\RR^d$ and define $\delta_t = \|\theta_t - \theta_t'\| $  where $\theta_t$, $ \theta_t'$ are defined recursively through 
    \begin{equation*}
        \theta_{t+1} = G_{\eta_t,z_t}(\theta_t), \theta_{t+1}' = G_{\eta'_t,z'_t}(\theta'_t).
    \end{equation*}
Then, we have the recurrence relation
\begin{align*}
    \delta_0 & = 0 \\
    \delta_{t+1} &\leq \left\{
        \begin{array}{ll}
            \nu\delta_t & \mbox{if } G_{\eta_t,z_t} = G_{\eta_t,z'_t} \mbox{ is } \nu\mbox{-expansive} \\
            \min{\{1,\nu\}}\delta_t + 2\eta_tL & \mbox{if } \ell \mbox{ is } L\mbox{-Lipschitz} \mbox{ and } G_{\eta_t,z_t} \mbox{ is } \nu\mbox{-expansive}
        \end{array}
    \right.
\end{align*}

\end{lemma}

\section{Proofs of Section \ref{sec:convex}}

\subsection{Theorem \ref{thme:convex}}
\label{app:convex}

First, notice that if $\min_k\{W_{kk}\} = 0$, by assumption we have $\eta_t\leq \frac{2\min_k\{W_{kk}\}}{\beta} = 0$. Hence our algorithm is perfectly stable and the bound is trivially obtained as $A(S)=\theta^{(0)}$ is data-independent. In the following, we focus on the case where $\min_k\{W_{kk}\} > 0$.

Thanks to Lemma \ref{lemma:ob-avg-gen}, we simply need to prove that $A(S)$ is on average $\varepsilon$-stable with $\varepsilon \leq \frac{2L \sum_{t=0}^{T-1}\eta_t}{mn}$. Taking the notations of Def. \ref{def:on-average}, we denote by $A_k(S)=\theta_k^{(T)}$, and $A_k(S^{(ij)})=\tilde{\theta}_k^{(T)}$, the outputs of agent $k$ for D-SGD (Variant~B) at round $T$, run over $S$ and $S^{(ij)}$ respectively. More generally, $\{\theta_k^{(t)}\}_{t=0}^T$ (respectively $\{\tilde{\theta}_k^{(t)}\}_{t=0}^T$), refer to the iterates of agent $k$ run over $S$ (respectively $S^{(ij)}$). We also denote by $\{Z'_{vk}\}_{vk}$ the elements of the data set $S^{(ij)}$, i.e. $Z'_{vk} = Z_{vk}$ for $(v,k) \neq (i,j)$ and $Z'_{ij} = \TZ_{ij} \neq Z_{ij}$.

    For all $k = 1,\ldots,m$ and $t\geq 1$, we have

    \begin{align}
        \|\theta_k^{(t+1)} - & \tilde{\theta}_k^{(t+1)}\|_2   = \Big\|\sum^m_{l=1}W_{kl}\theta_l^{(t)} - \eta_t\nabla \ell(\theta^{(t)}_k;Z_{I^t_kk}) - \sum^m_{l=1}W_{kl}\tilde{\theta}_l^{(t)} + \eta_t\nabla \ell(\tilde{\theta}^{(t)}_k;Z'_{I^t_kk})\Big\|_2 \nonumber  \\
        & = \Big\|W_{kk}\Big(\theta_k^{(t)}- \frac{\eta_t}{W_{kk}}\nabla \ell(\theta^{(t)}_k;Z_{I^t_kk}) - \tilde{\theta}_k^{(t)} + \frac{\eta_t}{W_{kk}}\nabla \ell(\tilde{\theta}^{(t)}_k;Z'_{I^t_kk})\Big) + \sum^m_{l\neq k}W_{kl}(\theta_l^{(t)}-\tilde{\theta}_l^{(t)})  \Big\|_2 \nonumber \\
        & \leq W_{kk}\Big\|\theta_k^{(t)}- \frac{\eta_t}{W_{kk}}\nabla \ell(\theta^{(t)}_k;Z_{I^t_kk}) - \tilde{\theta}_k^{(t)} + \frac{\eta_t}{W_{kk}}\nabla \ell(\tilde{\theta}^{(t)}_k;Z'_{I^t_kk})\Big\|_2 + \sum^m_{l\neq k}W_{kl}\Big\|\theta_l^{(t)}-\tilde{\theta}_l^{(t)}  \Big\|_2
        \label{eq:almost_sure_convex}
    \end{align}

Thanks to Lemma~\ref{lem:expansivity} (part. 2) in
Appendix~\ref{app:lemmas} and the fact that, by assumption, $\forall k, \frac{\eta_t}{W_{kk}} \leq \frac{2}{\beta}$, the update rules $G_{\eta_t,Z_{I^t_kk}}(\theta)=\theta -\frac{\eta_t}{W_{kk}}\nabla\ell(\theta;Z_{I^t_kk})$ and $G_{\eta_t,Z'_{I^t_kk}}(\theta)=\theta -\frac{\eta_t}{W_{kk}}\nabla\ell(\theta;Z'_{I^t_kk})$ are $1$-expansive. Hence:

\fbox{If $k \neq j$,} we have $Z_{I^t_kk} = Z'_{I^t_kk}$, which gives from Eq. \eqref{eq:almost_sure_convex} and Lemma \ref{lem:growth} that:

\begin{align}
    \|\theta_k^{(t+1)} - \tilde{\theta}_k^{(t+1)}\|_2  & \leq W_{kk}\Big\|\theta_k^{(t)}- \frac{\eta_t}{W_{kk}}\nabla \ell(\theta^{(t)}_k;Z_{I^t_kk}) - \tilde{\theta}_k^{(t)} + \frac{\eta_t}{W_{kk}}\nabla \ell(\tilde{\theta}^{(t)}_k;Z_{I^t_kk})\Big\|_2 + \sum^m_{l\neq k}W_{kl}\Big\|\theta_l^{(t)}-\tilde{\theta}_l^{(t)}  \Big\|_2 \nonumber \\
    &  \leq W_{kk}\Big\|\theta_k^{(t)} - \tilde{\theta}_k^{(t)} \Big\|_2 + \sum^m_{l\neq k}W_{kl}\Big\|\theta_l^{(t)}-\tilde{\theta}_l^{(t)}  \Big\|_2 \nonumber \\
    & \leq \sum^m_{l=1}W_{kl}\Big\|\theta_l^{(t)}-\tilde{\theta}_l^{(t)}  \Big\|_2
    \label{eq:recurs_k}
\end{align}

\fbox{If $k = j$:} \\

\underline{With probability $1-\frac{1}{n}$}, $I^t_j\neq i$ so $Z_{I^t_jj} = Z'_{I^t_jj}$ and we therefore have again the relation of Equation \eqref{eq:recurs_k}.

%$Z_{I^t_ll} = Z'_{I^t_ll}$ for all $l=1,\ldots m$. It results that \underline{with probability $1-\frac{1}{n}$}:
%\begin{align*}
%    \frac{1}{m} \sum^m_{l=1}\Big\|\theta_l^{(t)} - \eta_t\nabla \ell(\theta^{(t)}_l;Z_{I^t_ll}) & - \tilde{\theta}_l^{(t)} + \eta_t\nabla \ell(\tilde{\theta}^{(t)}_l;Z'_{I^t_ll})\Big\|_2 \\
%    & = \frac{1}{m} \sum^m_{l=1}\Big\|\theta_l^{(t)} - \eta_t\nabla \ell(\theta^{(t)}_l;Z_{I^t_ll}) - \tilde{\theta}_l^{(t)} + \eta_t\nabla \ell(\tilde{\theta}^{(t)}_l;Z_{I^t_ll})\Big\|_2 \\
%    & \hspace{-0.25cm} \overset{Lem. \ref{lem:growth}}{\leq}\frac{1}{m} \sum^m_{l=1}\Big\|\theta_l^{(t)} - \tilde{\theta}_l^{(t)} \Big\|_2  = \dbar_t \;.
%\end{align*}

\underline{With probability $\frac{1}{n}$} however, $I^t_j=i$ and in that case $Z_{I^t_jj} = Z_{ij}\neq \TZ_{ij} =  Z'_{I^t_jj} $. 
With probability $\frac{1}{n}$, we therefore have: 

\begin{align}
    \|\theta_j^{(t+1)} - \tilde{\theta}_j^{(t+1)}\|_2  & \leq W_{jj}\Big\|\theta_j^{(t)}- \frac{\eta_t}{W_{jj}}\nabla \ell(\theta^{(t)}_j;Z_{ij}) - \tilde{\theta}_j^{(t)} + \frac{\eta_t}{W_{jj}}\nabla \ell(\tilde{\theta}^{(t)}_j;\TZ_{ij})\Big\|_2 + \sum^m_{l\neq j}W_{jl}\Big\|\theta_l^{(t)}-\tilde{\theta}_l^{(t)}  \Big\|_2 \label{eq:starting-point-lemma}  \\
    & \hspace{-0.2cm} \overset{Lem. \ref{lem:growth}}{\leq}  \hspace{-0.2cm} W_{jj}\Big(\|\theta_j^{(t)}-\tilde{\theta}^{(t)}_j\|_2 + \frac{2\eta_tL}{W_{jj}}\Big)+ \sum^m_{l\neq j}W_{jl}\Big\|\theta_l^{(t)}-\tilde{\theta}_l^{(t)}  \Big\|_2 \nonumber  \\
    & = \sum^m_{l=1}W_{jl}\Big\|\theta_l^{(t)}-\tilde{\theta}_l^{(t)}  \Big\|_2 + 2\eta_tL \nonumber
\end{align}

Denoting by $\delta^{(t)}_k(i,j)\triangleq \|\theta_k^{(t)} - \tilde{\theta}_k^{(t)}\|_2$ and combining previous results, we get for all $k=1,\ldots,m$, the recursion: 
\begin{align*}
    \EE[\delta^{(T)}_k(i,j)] \leq \sum^m_{l=1}W_{kl} \EE[\delta^{(T-1)}_l(i,j)]  + \frac{2\eta_tL}{n}\mathds{1}_{\{k=j\}}\;.
\end{align*}

In vector format, this writes $\EE[\delta^{(T)}(i,j)] \leq W \EE[\delta^{(T-1)}(i,j)]  + \frac{2\eta_tL}{n}e_j$ (the inequality is meant coordinate-wise), where $\delta^{(t)}(i,j)$ contains the values of $\delta_k^{(t)}(i,j)$, $\forall k$ and $e_j$ is the $j$-th element of the canonical basis. Unrolling this recursion until $t=0$, and noticing that $\delta_k^{(0)}(i,j)=0$, we get:
\begin{equation}
\label{eq:final-convex-worst}
\EE[\delta^{(T)}(i,j)] \leq \frac{2L}{n}\sum_{t=0}^{T-1}W^{T-t-1}e_j\eta_t \Longrightarrow \EE[\delta_k^{(T)}(i,j)] \leq \frac{2L}{n}\sum_{t=0}^{T-1}(W^{T-t-1})_{kj}\eta_t
\end{equation}

 Averaging over $i$ and $j$ and using the fact that the power of $W$ is also doubly stochastic, i.e. $\sum_j(W^{T-t-1})_{kj} = 1$, we obtain that the on-average model stability is upper bounded by $\frac{2L \sum_{t=0}^{T-1}\eta_t}{mn}$, which concludes the proof with a direct application of Lemma \ref{lemma:ob-avg-gen}.

\hfill\qedsymbol{}

\subsection{Theorem \ref{thme:strongly}}
\label{app:strongly-convex}

Like for convex functions, if $\min_k\{W_{kk}\} = 0$ we have $\eta_t  = 0$, and the bound is trivially obtained. In the following, we therefore focus on the case where $\min_k\{W_{kk}\} > 0$.

The proof is analogous to the one obtained for the general convex case 
(Theorem \ref{thme:convex}). We keep the same notations, where $
\EE\delta^{(T)}_k(i,j)= \EE\|\theta_k^{(T)} - \tilde{\theta}_k^{(T)}\|_2$ is the quantity we wish to control, on average over $i$ and $j$. Using the fact that the Euclidean projection $\Pi_\Theta$ is 1-expansive  (see e.g. Lemma 4.6 in \citet{hardt2016train}), we can directly obtain Equation \eqref{eq:almost_sure_convex} using the same arguments.

    Thanks to Lemma \ref{lem:expansivity} (part. 3), we notice that for all
    $k$, the update rules $G_{\eta,Z_{I^t_kk}}(\theta)=\theta -\frac{\eta}{W_{kk}}\nabla\ell
    (\theta;Z_{I^t_kk})$ and $G_{\eta,Z'_{I^t_kk}}(\theta)=\theta
    -\frac{\eta}{W_{kk}}\nabla\ell(\theta;Z'_{I^t_kk})$ are $(1-\frac{\eta\mu}
    {2W_{kk}})$-expansive. Indeed, since we always have $\mu\leq\beta$ and by assumption $\frac{\eta}{W_{kk}}\leq
    \frac{1}{\beta}\leq \frac{2}{\beta+\mu}$, we can apply the lemma and then use the fact that $\frac{\eta\beta\mu}{W_{kk}(\beta+\mu)}\geq
    \frac{\eta\beta\mu}{2W_{kk}\beta}= \frac{\eta\mu}{2W_{kk}}$. We now follow the proof of the convex case by splitting the analysis similarly.

    \fbox{If $k \neq j$,} we have $Z_{I^t_kk} = Z'_{I^t_kk}$, which gives from Eq. \eqref{eq:almost_sure_convex} and Lemma \ref{lem:growth} with the $(1-\frac{\eta\mu}
    {2W_{kk}})$-expansivity, that:
    \begin{align}
        \|\theta_k^{(t+1)} - \tilde{\theta}_k^{(t+1)}\|_2  & \leq W_{kk} \Big(1-\frac{\eta\mu}
    {2W_{kk}}\Big)\Big\|\theta_k^{(t)}-\tilde{\theta}_k^{(t)}  \Big\|_2 + \sum^m_{l\neq j }W_{kl}\Big\|\theta_l^{(t)}-\tilde{\theta}_l^{(t)}  \Big\|_2 \nonumber \\
    & = \sum^m_{l=1 }W_{kl}\Big\|\theta_l^{(t)}-\tilde{\theta}_l^{(t)}  \Big\|_2
    - \frac{\eta\mu}{2}\Big\|\theta_k^{(t)}-\tilde{\theta}_k^{(t)}  \Big\|_2
        \label{eq:recurs_k_strongly}
    \end{align}

\fbox{If $k = j$:} \\

\underline{With probability $1-\frac{1}{n}$}, $I^t_j\neq i$ so $Z_{I^t_jj} = Z'_{I^t_jj}$ and we therefore have again the relation of Eq. \eqref{eq:recurs_k_strongly}.

\underline{With probability $\frac{1}{n}$} however, $I^t_j=i$ and in that case $Z_{I^t_jj} = Z_{ij}\neq \TZ_{ij} =  Z'_{I^t_jj} $. 
With probability $\frac{1}{n}$, we therefore have: 

\begin{align}
    \|\theta_j^{(t+1)} - \tilde{\theta}_j^{(t+1)}\|_2  & \leq W_{jj}\Big\|\theta_j^{(t)}- \frac{\eta}{W_{jj}}\nabla \ell(\theta^{(t)}_j;Z_{ij}) - \tilde{\theta}_j^{(t)} + \frac{\eta}{W_{jj}}\nabla \ell(\tilde{\theta}^{(t)}_j;\TZ_{ij})\Big\|_2 + \sum^m_{l\neq j}W_{jl}\Big\|\theta_l^{(t)}-\tilde{\theta}_l^{(t)}  \Big\|_2 \nonumber  \\
    & \hspace{-0.2cm} \overset{Lem. \ref{lem:growth}}{\leq}  \hspace{-0.2cm} W_{jj}\Big(\big(1-\frac{\eta\mu}{2W_{jj}}\big)\|\theta_j^{(t)}-\tilde{\theta}^{(t)}_j\|_2 + \frac{2\eta L}{W_{jj}}\Big)+ \sum^m_{l\neq j}W_{jl}\Big\|\theta_l^{(t)}-\tilde{\theta}_l^{(t)}  \Big\|_2 \nonumber  \\
    & = \sum^m_{l=1}W_{jl}\Big\|\theta_l^{(t)}-\tilde{\theta}_l^{(t)}  \Big\|_2 - \frac{\eta\mu}{2}\Big\|\theta_j^{(t)}-\tilde{\theta}_j^{(t)}  \Big\|_2 + 2\eta L \nonumber
\end{align}

Combining previous results, we get for all $k=1,\ldots,m$, the recursion: 
\begin{align*}
    \EE[\delta^{(T)}_k(i,j)] \leq \sum^m_{l=1}W_{kl} \EE[\delta^{(T-1)}_l(i,j)] - \frac{\eta\mu}{2}\EE[\delta^{(T-1)}_k(i,j)] + \frac{2\eta L}{n}\mathds{1}_{\{k=j\}}\;.
\end{align*}

In vector format, this writes (the inequality is meant coordinate-wise) $$\EE[\delta^{(T)}(i,j)] \leq \Big(W - \frac{\eta\mu}{2}I\Big)\EE[\delta^{(T-1)}(i,j)]  + \frac{2\eta_tL}{n}e_j\;,$$ where $\delta^{(T)}(i,j)$ contains the values of $\delta_k^{(T)}(i,j)$, $\forall k$ and $e_j$ is the $j$-th element of the canonical basis. Unrolling this recursion until $t=0$, and noticing that $\delta_k^{(0)}(i,j)=0$, we get:
\begin{equation*}
\EE[\delta^{(T)}(i,j)] \leq \frac{2\eta L}{n}\sum_{t=0}^{T-1}\Big(W - \frac{\eta\mu}{2}I\Big)^{t}e_j\;.
\end{equation*}
 Averaging the previous equation over $i$ and $j$ and using the fact that $\sum_{j=1}^m e_j=\1$, we have

 \begin{equation*}
     \frac{1}{mn}\sum_{i=1}^n\sum_{j=1}^m\EE[\delta^{(T)}(i,j)] \leq \frac{2\eta L}{mn}\sum_{t=0}^{T-1}\Big(W - \frac{\eta\mu}{2}I\Big)^{t}\1\;.
 \end{equation*}

 Since $(W-\frac{\eta\mu}{2}I)\1=(1-\frac{\eta\mu}{2})\1$, by induction we have $(W-\frac{\eta\mu}{2}I)^t\1=(1-\frac{\eta\mu}{2})^t\1$ and we can finally get $\forall k$:
\begin{align}
\label{eq:final-worst-strongly}
    \frac{1}{mn}\sum_{i=1}^n\sum_{j=1}^m\EE[\delta_k^{(T)}(i,j)] \leq \frac{2\eta L}{mn}\sum_{t=0}^{T-1}\Big(1 - \frac{\eta\mu}{2}\Big)^{t} \leq \frac{4L}{\mu mn}\;,
\end{align}

 which makes $A_k$ on average $\varepsilon$-stable (Def. \ref{def:on-average}) with $\varepsilon = \frac{4L}{\mu mn}$. Like in the convex case, a direct application of Lemma \ref{lemma:ob-avg-gen} completes the proof.

 \hfill\qedsymbol{}

\section{Proofs of Section \ref{sec:non-convex}}

\subsection{Theorem \ref{thme:non-convex}}
\label{app:non-convex}

Our analysis for the non-convex case relies on on-average model stability and leverages the fact that D-SGD can make
several steps before using the one example that has been swapped. This idea is
summarized in the following lemma.

\begin{lemma} \label{lemma:key-non-conv} Assume that the loss function $\ell
(\cdot,z)$ is nonnegative and $L$-Lipschitz for all $z$. For all $i=1,\ldots,n$ and $j=1,\ldots, m$, let $\{\theta_k^{(t)}\}_{t=0}^T$ and $ \{\tilde{\theta}_k^{(t)}(i,j)\}_{t=0}^T$, the iterates of agent $k = 1,\ldots, m$ for D-SGD (Variant A and B) run on $S$ and $S^{(ij)}$ respectively. Then, for every $t_0 \in \{0, 1, \ldots, T\}$ we have:
    \begin{equation}
        |\EE_{A,S}[R(A_k(S)) - R_S(A_k(S))]| \leq \frac{t_0}{n}\sup_{\theta, z} \ell(\theta;z) + \frac{L}{mn}\sum_{i=1}^n\sum_{j=1}^m\EE[\delta_k^{(T)}(i,j)\big| \delta^{(t_0)}(i,j) = \mathbf{0}] 
    \end{equation}

where $\delta^{(t)}(i,j)$ is the vector containing $\forall k=1,\ldots,m$, $\delta_k^{(t)}(i,j) = \|\theta_k^{(t)} - \tilde{\theta}_k^{(t)}(i,j)\|_2$.
\end{lemma}

\begin{proof}
    Consider the notation of Def. \ref{def:on-average} and notice that 

$$R(A_k(S)) = \frac{1}{m}\sum_{j=1}^m\EE_{Z \sim \calD_j} [\ell(A_k(S); Z)] = \frac{1}{mn}\sum_{j=1}^m\sum_{j=1}^n\EE_{\TS} [\ell(A_k(S); \TZ_{ij})].$$ Then, for all $k = 1,\ldots,m$, by linearity of expectation we have

\begin{align}
    \EE_{A,S}[R(A_k(S)) - R_S(A_k(S))] & = \EE_{A,S,\TS}\Bigg[\frac{1}{mn}\sum_{j=1}^m\sum_{i=1}^n \Big(\ell(A_k(S); \TZ_{ij}) - \ell(A_k(S); Z_{ij})\Big)\Bigg] \nonumber \\
    & = \EE_{A,S,\TS}\Bigg[\frac{1}{mn}\sum_{j=1}^m\sum_{i=1}^n \Big(\ell(A_k(S^{(ij)}); Z_{ij}) - \ell(A_k(S); Z_{ij})\Big)\Bigg].\nonumber
\end{align}

Hence,
\begin{align*}
    |\EE_{A,S}[R(A_k(S)) - R_S(A_k(S))]| & \leq \EE_{A,S,\TS}\Bigg[\frac{1}{mn}\sum_{j=1}^m\sum_{i=1}^n \Big|\ell(A_k(S^{(ij)}); Z_{ij}) - \ell(A_k(S); Z_{ij})\Big|\Bigg]  \\
    & = \frac{1}{mn}\sum_{j=1}^m\sum_{i=1}^n \EE_{A,S,\TS}\Bigg[ \Big|\ell(A_k(S^{(ij)}); Z_{ij}) - \ell(A_k(S); Z_{ij})\Big|\Bigg]
\end{align*}

Let the event $\calE (i,j) = \{\delta^{(t_0)}(i,j) = \mathbf{0}\} $, we have $\forall i,j$:

    \begin{align*}
        \EE_{A,S,\TS}& \Big[ \big|\ell(A_k(S^{(ij)}); Z_{ij})  - \ell(A_k(S); Z_{ij})\big|\Big] \\
        & = \IP(\mathcal{E}(i,j))\EE[|\ell(A_k(S^{(ij)}); Z_{ij}) - \ell(A_k(S); Z_{ij})| \big| \mathcal{E}(i,j)]  \\
        & \hspace{4cm}+  \IP(\mathcal{E}(i,j)^c)\EE[|\ell(A_k(S^{(ij)}); Z_{ij}) - \ell(A_k(S); Z_{ij})| \big| \mathcal{E}(i,j)^c] \\
        & \leq \EE[|\ell(A_k(S^{(ij)}); Z_{ij}) - \ell(A_k(S); Z_{ij})| \big| \mathcal{E}(i,j)] +\IP(\mathcal{E}(i,j)^c)\cdot \sup_{\theta,z}\ell(\theta;z) \\ 
        & \leq L\EE[\|A_k(S)  - A_k(S^{(ij)})\| \big| \mathcal{E}(i,j)]+\IP(\mathcal{E}(i,j)^c)\cdot \sup_{\theta,z}\ell(\theta;z) \\
        & = L\EE[\delta_k^{(T)}(i,j) \big| \mathcal{E}(i,j)]+\IP(\mathcal{E}(i,j)^c)\cdot \sup_{\theta,z}\ell(\theta;z)
    \end{align*}

    It remains to bound $\IP(\mathcal{E}(i,j)^c)$. Let $T_0$ be the random variable of the first time step D-SGD uses the swapped example. Since we necessarily have $\{T_0 > t_0\} \subset \mathcal{E}(i,j)$, we have $ \mathcal{E}(i,j)^c \subset \{T_0 \leq t_0\}$ and therefore $\IP(\mathcal{E}(i,j)^c) \leq \IP(T_0 \leq t_0) = \sum_{t=1}^{t_0}\IP(T_0=t)\leq \sum_{t=1}^{t_0}\frac{1}{n} = \frac{t_0}{n}$. Averaging over $i$ and $j$ completes the proof.

\end{proof}

We can now move on to the proof of the main theorem. We first apply Lemma \ref{lemma:key-non-conv} and the fact that, by assumption, $\ell \in [0,1]$, so that for any $t_0 \in \{0, 1, \ldots, T\}$ and any $k=1,\ldots, m$, we have:

\begin{equation}
\label{eq:right-hand}
    |\EE_{A,S}[R(A_k(S)) - R_S(A_k(S))]| \leq \frac{t_0}{n}+ \frac{L}{mn}\sum_{i=1}^n\sum_{j=1}^m\EE[\delta_k^{(T)}(i,j)\big| \delta^{(t_0)}(i,j) = \mathbf{0}] 
\end{equation}

It remains to control the right-hand term of Equation \eqref{eq:right-hand}. We start with the proof for the variant B of D-SGD. The proof for Variant A will follow.

\textbf{Variant B:}

For a fixed couple $(i,j)$, we are first going to control the vector $\Delta^{(t)}(i,j) \triangleq\EE[\delta^{(t)}(i,j) | \delta^{(t_0)}(i,j) = \mathbf{0}]$. When it is clear from context, we simply write $\tilde{\theta}_k^{(t)}(i,j) = \tilde{\theta}_k^{(t)}$. The proof is analogous to the one obtained for convex cases 
(Theorem \ref{thme:convex} and \ref{thme:strongly}) and we can start directly from Equation \eqref{eq:almost_sure_convex} using the same arguments.

    Thanks to Lemma \ref{lem:expansivity} (part. 3), we notice that for all
    $k$, the update rules $G_{\eta_t,Z_{I^t_kk}}(\theta)=\theta -\frac{\eta_t}{W_{kk}}\nabla\ell
    (\theta;Z_{I^t_kk})$ and $G_{\eta_t,Z'_{I^t_kk}}(\theta)=\theta
    -\frac{\eta_t}{W_{kk}}\nabla\ell(\theta;Z'_{I^t_kk})$ are $(1+\frac{\eta_t\beta}
    {W_{kk}})$-expansive. Following the proof of the convex cases, we can split the analysis similarly.

    \fbox{If $k \neq j$,} we have $Z_{I^t_kk} = Z'_{I^t_kk}$, which gives from Eq. \eqref{eq:almost_sure_convex} and Lemma \ref{lem:growth} with the $(1+\frac{\eta_t\beta}
    {W_{kk}})$-expansivity, that:
    \begin{align}
        \|\theta_k^{(t+1)} - \tilde{\theta}_k^{(t+1)}\|_2  & \leq W_{kk} \Big(1+\frac{\eta_t\beta}
    {W_{kk}}\Big)\Big\|\theta_k^{(t)}-\tilde{\theta}_k^{(t)}  \Big\|_2 + \sum^m_{l\neq j }W_{kl}\Big\|\theta_l^{(t)}-\tilde{\theta}_l^{(t)}  \Big\|_2 \nonumber \\
    & = \sum^m_{l=1 }W_{kl}\Big\|\theta_l^{(t)}-\tilde{\theta}_l^{(t)}  \Big\|_2
    + \eta_t\beta\Big\|\theta_k^{(t)}-\tilde{\theta}_k^{(t)}  \Big\|_2
        \label{eq:recurs_k_nonconv}
    \end{align}

\fbox{If $k = j$:} \\

\underline{With probability $1-\frac{1}{n}$}, $I^t_j\neq i$ so $Z_{I^t_jj} = Z'_{I^t_jj}$ and we therefore have again the relation of Eq. \eqref{eq:recurs_k_nonconv}.

\underline{With probability $\frac{1}{n}$} however, $I^t_j=i$ and in that case $Z_{I^t_jj} = Z_{ij}\neq \TZ_{ij} =  Z'_{I^t_jj} $. 
With probability $\frac{1}{n}$, we therefore have: 

\begin{align}
    \|\theta_j^{(t+1)} - \tilde{\theta}_j^{(t+1)}\|_2  & \leq W_{jj}\Big\|\theta_j^{(t)}- \frac{\eta_t}{W_{jj}}\nabla \ell(\theta^{(t)}_j;Z_{ij}) - \tilde{\theta}_j^{(t)} + \frac{\eta_t}{W_{jj}}\nabla \ell(\tilde{\theta}^{(t)}_j;\TZ_{ij})\Big\|_2 + \sum^m_{l\neq j}W_{jl}\Big\|\theta_l^{(t)}-\tilde{\theta}_l^{(t)}  \Big\|_2 \nonumber  \\
    & \hspace{-0.2cm} \overset{Lem. \ref{lem:growth}}{\leq}  \hspace{-0.2cm} W_{jj}\Big(\big(1+\frac{\eta_t\beta}
    {W_{jj}}\big)\|\theta_j^{(t)}-\tilde{\theta}^{(t)}_j\|_2 + \frac{2\eta_t L}{W_{jj}}\Big)+ \sum^m_{l\neq j}W_{jl}\Big\|\theta_l^{(t)}-\tilde{\theta}_l^{(t)}  \Big\|_2 \nonumber  \\
    & = \sum^m_{l=1}W_{jl}\Big\|\theta_l^{(t)}-\tilde{\theta}_l^{(t)}  \Big\|_2 + \eta_t\beta\Big\|\theta_j^{(t)}-\tilde{\theta}_j^{(t)}  \Big\|_2 + 2\eta_t L \nonumber
\end{align}

From the previous equations, we get that $\Delta^{(t+1)}(i,j) \leq (W + \eta_t\beta I)\Delta^{(t)}(i,j) + \frac{2\eta_t L}{n}e_j$ (the inequality, and the following ones are meant coordinate-wise). Let $\Delta^{(t)}=\frac{1}{mn}\sum_{i,j}\Delta^{(t)}(i,j)$, then using the fact that $\eta_t\leq \frac{c}{t+1}$, $c>0$, we have $\forall t\geq t_0$:

\begin{align*}
    \Delta^{(t+1)} & \leq (W + \eta_t\beta I)\Delta^{(t)} + \frac{2\eta_t L}{mn}\1 \\
    & \leq (W + \frac{c \beta}{t+1} I)\Delta^{(t)} + \frac{2c L}{mn(t+1)}\1
\end{align*}

Since $\Delta^{(t_0)} = \mathbf{0}$, we can unroll the previous recursion from $T$ to $t_0+1$ and get:

\begin{align*}
    \Delta^{(T)} & \leq \frac{2cL}{Tmn}\1 + \sum_{t=t_0+1}^{T-1}\Big\{\prod_{\tau=t+1}^T \Big(W + \frac{c \beta}{\tau} I\Big)\Big\}\frac{2cL}{tmn}\1 \nonumber \\ 
    & = \frac{2cL}{Tmn}\1 + \sum_{t=t_0+1}^{T-1}\Big\{\prod_{\tau=t+1}^T \Big(1 + \frac{c \beta}{\tau} \Big)\Big\}\frac{2cL}{tmn}\1\;, \nonumber 
\end{align*}

where in the last equality we used the fact that $(W + \frac{c \beta}{\tau} I)\1 = (1+ \frac{c \beta}{\tau})\1$, which by induction gives $\prod_{\tau}\Big(W + \frac{c \beta}{\tau} I\Big)\1 = \prod_{\tau} \Big(1 + \frac{c \beta}{\tau} \Big)\1$. Then, we focus on the coordinate of interest $k$ and using the fact that $1+x\leq \exp(x)$, we have:
\begin{align}
    \Delta_k^{(T)} & \leq \frac{2cL}{Tmn} + \sum_{t=t_0+1}^{T-1}\Big\{\prod_{\tau=t+1}^T \exp \Big(\frac{c\beta}{\tau}\Big)\Big\}\frac{2cL}{tmn} \nonumber \\ 
    & = \frac{2cL}{Tmn} + \sum_{t=t_0+1}^{T-1}\exp \Big(c\beta\sum_{\tau=t+1}^T \frac{1}{\tau}\Big)\frac{2cL}{tmn} \nonumber \\
    & \leq \frac{2cL}{Tmn} + \sum_{t=t_0+1}^{T-1}\exp \Big(c\beta \log\big(\frac{T}{t}\big)\Big)\frac{2cL}{tmn} \nonumber \\
    & = \frac{2cL}{Tmn} + \sum_{t=t_0+1}^{T-1}\Big(\frac{T}{t}\Big)^{c\beta}\frac{2cL}{tmn}\nonumber \\
    & = \frac{2cLT^{c\beta}}{T^{c\beta +1}mn} + \frac{2cLT^{c\beta}}{mn}\sum_{t=t_0+1}^{T-1} t^{-c\beta-1} \nonumber \\
    & = \frac{2cLT^{c\beta}}{mn}\sum_{t=t_0+1}^{T} t^{-c\beta-1} \nonumber \\
    & \leq \frac{2cLT^{c\beta}}{mn} \frac{t_0^{-\beta c}}{c\beta} = \frac{2L}{\beta mn}\Big(\frac{T}{t_0}\Big)^{c \beta}, \label{eq:recursion-non-convex}
\end{align}
where the last inequality is obtained using bounds over (partial) harmonic series.

Plugging this result into \eqref{eq:right-hand}, we obtain
\begin{equation*}
    |\EE_{A,S}[R(A_k(S)) - R_S(A_k(S))]| \leq \frac{t_0}{n} + \frac{2L^2}{\beta mn}\Big(\frac{T}{t_0}\Big)^{c \beta}.
\end{equation*}

Then, taking $t_0 = \Big(\frac{2L^2c}{m}\Big)^{\frac{1}{c\beta + 1}}T^{\frac{c\beta}{c\beta + 1}}$ (approximate minimizer of the right-hand term above), we have 
\begin{equation*}
    |\EE_{A,S}[R(A_k(S)) - R_S(A_k(S))]| \leq (1+\frac{1}{\beta c})(2cL^2)^{\frac{1}{\beta c +1}}\frac{T^{\frac{\beta c}{\beta c +1}}}{m^{\frac{1}{\beta c +1}}n}\;,
\end{equation*}
which concludes the proof for Variant B.

\begin{remark}
Note that the inequality \eqref{eq:recursion-non-convex} is not optimal in the specific case $t_0 = 0$ (it diverges) and somehow prevents us from taking $t_0=0$ in the proof. However, taking this value could be optimal in some regimes. Hence, the analysis and our final bound could be slightly improved by looking at the minimum between the cases where $t_0=0$ or not. When taking $t_0=0$, the bound in Eq. \eqref{eq:recursion-non-convex} can be improved to $\frac{2cLT^{c\beta}}{mn}(1+\frac{1}{c\beta})$.   
\end{remark}

\textbf{Variant A:}

The proof for the variant A is essentially the same, where instead of Equation \eqref{eq:almost_sure_convex}, we have:

    \begin{align}
        \|\theta_k^{(t+1)} -  \tilde{\theta}_k^{(t+1)}\|_2  & = \Big\|\sum^m_{l=1}W_{kl}\Big(\theta_l^{(t)} - \eta_t\nabla \ell(\theta^{(t)}_l;Z_{I^t_ll})\Big) - \sum^m_{l=1}W_{kl}\Big(\tilde{\theta}_l^{(t)} + \eta_t\nabla \ell(\tilde{\theta}^{(t)}_l;Z'_{I^t_ll})\Big)\Big\|_2 \nonumber  \\
        & = \sum^m_{l=1}W_{kl}\Big\|\theta_l^{(t)} - \eta_t\nabla \ell(\theta^{(t)}_l;Z_{I^t_ll}) + \tilde{\theta}_l^{(t)} + \eta_t\nabla \ell(\tilde{\theta}^{(t)}_l;Z'_{I^t_ll})\Big\|_2
        \label{eq:almost_sure-nonconvex}
    \end{align}

Since, thanks to Lemma \ref{lem:expansivity} (part. 1), the update rules $G_{\eta_t,Z_{I^t_ll}}(\theta)=\theta -\eta_t\nabla\ell(\theta;Z_{I^t_ll})$ and $G_{\eta_t,Z'_{I^t_ll}}(\theta)=\theta -\eta_t\nabla\ell(\theta;Z'_{I^t_ll})$ are $(1+\eta_t\beta)$-expansive, we can use the same arguments as before to show that

\begin{equation*}
    \EE[\|\theta_l^{(t)} - \eta_t\nabla \ell(\theta^{(t)}_l;Z_{I^t_ll}) + \tilde{\theta}_l^{(t)} + \eta_t\nabla \ell(\tilde{\theta}^{(t)}_l;Z'_{I^t_ll})\| \big| \mathcal{F}_t] \leq (1+\eta_t\beta)\|\theta_l^{(t)} - \tilde{\theta}_l^{(t)}\| + \frac{2L\eta_t}{n}\mathds{1}_{\{l=j\}}\;,
\end{equation*}
where $\mathcal{F}_t$ is the natural filtration at time $t$.

Combining previous equations and the notation of the proof of Variant B, we get the vector format relation: 
\begin{align*}
    \Delta^{(t+1)} & \leq (1 + \eta_t\beta)W\Delta^{(t)} + \frac{2\eta_t L}{mn}\1 \\
    & \leq (1 + \frac{c \beta}{t+1} )W\Delta^{(t)} + \frac{2c L}{mn(t+1)}\1
\end{align*}

Since $\Delta^{(t_0)} = \mathbf{0}$, we can unroll the previous recursion from $T$ to $t_0+1$ and get:
\begin{align*}
    \Delta^{(T)} & \leq \frac{2cL}{Tmn}\1 + \sum_{t=t_0+1}^{T-1}\Big\{\prod_{\tau=t+1}^T \Big(1 + \frac{c \beta}{\tau}\Big)W\Big\}\frac{2cL}{tmn}\1 \nonumber \\ 
    & = \frac{2cL}{Tmn}\1 + \sum_{t=t_0+1}^{T-1}\Big\{\prod_{\tau=t+1}^T \Big(1 + \frac{c \beta}{\tau} \Big)\Big\}\frac{2cL}{tmn}\1\;, \nonumber 
\end{align*}

where in the last equality we used the fact that $(1 + \frac{c \beta}{\tau})W\1 = (1+ \frac{c \beta}{\tau})\1$.  From this point, the proof is the same as the one of Variant B, starting from the beginning of the derivation of Equation \eqref{eq:recursion-non-convex}.

\hfill\qedsymbol{}

\section{Proofs of Section \ref{sec:data-dep}}
\label{app:data-dep}

Like in the proof of Theorem \ref{thme:convex}, we simply need to consider the case where $\min_k\{W_{kk}\}>0$, the case $\min_k\{W_{kk}\}=0$ being trivial since it implies that $\eta_t=0$. 

\subsection{Lemma \ref{lemma:link-local-optim}}

We start the proof with the following lemma.

\begin{lemma}\emph{(Link with point-wise gradient norms).} \label{lemma:data-dep-nodewise} Under the same hypothesis as Theorem \ref{thme:convex}:
    \begin{equation*}
        |\EE_{A,S}[R(A_k(S)) - R_S(A_k(S))]| \leq \frac{2L}{mn}\sum_{t=0}^{T-1}\eta_t\sum_{j=1}^{m}(W^{T-t-1})_{kj}\frac{1}{n}\sum_{i=1}^{n}\EE[\|\nabla\ell(\theta_j^{(t)};Z_{ij})\|_2]
    \end{equation*}
\end{lemma}

\begin{proof}
    Until Equation \eqref{eq:starting-point-lemma}, the proof of Lemma \ref{lemma:data-dep-nodewise} is exactly the same as the one of Theorem \ref{thme:convex}. Let's start the proof from this point:

    If $k=j$, then with probability $\frac{1}{n}$ we have:

    \begin{align*}
        \|\theta_j^{(t+1)} - \tilde{\theta}_j^{(t+1)}\|_2  & \leq W_{jj}\Big\|\theta_j^{(t)}- \frac{\eta_t}{W_{jj}}\nabla \ell(\theta^{(t)}_j;Z_{ij}) - \tilde{\theta}_j^{(t)} + \frac{\eta_t}{W_{jj}}\nabla \ell(\tilde{\theta}^{(t)}_j;\TZ_{ij})\Big\|_2 + \sum^m_{l\neq j}W_{jl}\Big\|\theta_l^{(t)}-\tilde{\theta}_l^{(t)}  \Big\|_2  \\
        & \leq \sum^m_{l=1}W_{jl}\Big\|\theta_l^{(t)}-\tilde{\theta}_l^{(t)}  \Big\|_2 +\eta_t\Big\|\nabla \ell(\theta^{(t)}_j;Z_{ij})\Big\|_2 + \eta_t\Big\|\nabla \ell(\tilde{\theta}^{(t)}_j;\TZ_{ij})\Big\|_2 
    \end{align*}

Since the gradient norms $\|\nabla \ell(\theta^{(t)}_j;Z_{ij})\|_2$ and $\|\nabla \ell(\tilde{\theta}^{(t)}_j;\TZ_{ij})\|_2 $ have same law, they have the same expectation and we have:

\begin{equation*}
     \EE[\|\theta_j^{(t+1)} - \tilde{\theta}_j^{(t+1)}\|_2] \leq \sum^m_{l=1}W_{jl}\EE[\|\theta_l^{(t)}-\tilde{\theta}_l^{(t)}  \|_2] +\frac{2\eta_t}{n}\EE[\|\nabla \ell(\theta^{(t)}_j;Z_{ij})\|_2] 
\end{equation*}

Combining with the result for $k\neq j$ in the proof of Theorem \ref{thme:convex}, we have the following relation in vector format: 

\begin{align*}
    \EE[\delta^{(T)}(i,j)] & \leq W \EE[\delta^{(T-1)}(i,j)]  + \frac{2\eta_t}{n}\EE[\|\nabla \ell(\theta^{(t)}_j;Z_{ij})\|_2] \cdot e_j \\
    & \leq \frac{2}{n}\sum_{t=0}^{T-1}W^{T-t-1}\eta_t\EE[\|\nabla\ell(\theta^{(t)}_j;Z_{ij})\|_2] \cdot e_j
\end{align*}
where the second inequality is obtained by unrolling the recursion until $t=0$. For any agent $k=1,\ldots,n$, we therefore have 

\begin{equation}\label{eq:final-data-lemma}
\EE[\delta_k^{(T)}(i,j)]\leq \frac{2}{n}\sum_{t=0}^{T-1}(W^{T-t-1})_{kj}\eta_t\EE[\|\nabla\ell(\theta^{(t)}_j;Z_{ij})\|_2]    
\end{equation}

 Averaging over $i$ and $j$ and using Lemma \ref{lemma:ob-avg-gen} gives the final result.

\end{proof}

Going back to the proof of Lemma \ref{lemma:link-local-optim}, we can now average over $k$ the equation from Lemma \ref{lemma:data-dep-nodewise} and use the double stochasticity of $W^{T-t-1}$ to get 

\begin{align}
    & \frac{1}{m}\sum_{k=1}^m  |\EE_{A,S}[R(A_k(S)) - R_S(A_k(S))]| \leq \frac{2L}{mn}\sum_{t=0}^{T-1}\eta_t\frac{1}{mn}\sum_{i=1}^{n}\sum_{j=1}^{m}\EE[\|\nabla\ell(\theta_j^{(t)};Z_{ij})\|_2] \label{eq:nice-form} \\
    & \leq \frac{2L}{mn}\sum_{t=0}^{T-1}\eta_t\frac{1}{mn}\sum_{i=1}^{n}\sum_{j=1}^{m}\EE[\|\nabla\ell(\theta_j^{(t)};Z_{ij}) - \nabla R_{S_j}(\theta_j^{(t)})\|_2] + \frac{2L}{mn}\sum_{t=0}^{T-1}\eta_t\frac{1}{m}\sum_{j=1}^{m}\EE[\|\nabla R_{S_j}(\theta_j^{(t)})\|_2] \nonumber \\
    & =  \frac{2L}{mn}\sum_{t=0}^{T-1}\eta_t\frac{1}{mn}\sum_{i=1}^{n}\sum_{j=1}^{m}\EE\Big[\sqrt{\|\nabla\ell(\theta_j^{(t)};Z_{ij}) - \nabla R_{S_j}(\theta_j^{(t)})\|^2_2}\Big] + \frac{2L}{mn}\sum_{t=0}^{T-1}\eta_t\frac{1}{m}\sum_{j=1}^{m}\EE[\|\nabla R_{S_j}(\theta_j^{(t)})\|_2] \nonumber \\
    & \leq \frac{2L}{mn}\sum_{t=0}^{T-1}\eta_t\frac{1}{m}\sum_{j=1}^{m}\sqrt{\frac{1}{n}\sum_{i=1}^{n}\EE\Big[\|\nabla\ell(\theta_j^{(t)};Z_{ij}) - \nabla R_{S_j}(\theta_j^{(t)})\|^2_2\Big]} + \frac{2L}{mn}\sum_{t=0}^{T-1}\eta_t\frac{1}{m}\sum_{j=1}^{m}\EE[\|\nabla R_{S_j}(\theta_j^{(t)})\|_2]\;,\nonumber
\end{align}
where we used Jensen inequality in the last step. Using the fact that, by Assumption \ref{hyp:bounded-variance}, we have 
$\frac{1}{n}\sum_{i=1}^{n}\|\nabla\ell(\theta_j^{(t)};Z_{ij}) - \nabla R_{S_j}(\theta_j^{(t)})\|^2_2 \leq \sigma^2$ finishes the proof.

\hfill\qedsymbol{}

\subsection{Theorem \ref{thme:optim-dep}}
We start proving that under, the hypothesis of Theorem~\ref{thme:optim-dep}, all the eigenvalues of $W$  belong to $(-1,1]$.

\begin{lemma}
    \label{lem:second_eigenvalue}
    Let $W$ be an $n \times n$ symmetric, stochastic matrix with positive elements on the diagonal, then all the eigenvalues of $W$ are in $(-1,1]$.
\end{lemma}
\begin{proof}
As $W$ is symmetric and stochastic, then the module of its largest eigenvalue is equal to $1$ and all eigenvalues are in $[-1,1]$. We want now to prove that no eigenvalue can be equal to $-1$.
By an opportune permutation of the nodes, we can write the matrix $W$ as follows 
\[W=\left(\begin{array}{cccc}W_1 & 0_{n_1 \times n_2} & \cdots & 0_{n_1 \times n_C} \\ 
0_{n_2 \times n_1} & W_2 & \cdots & 0_{n_2 \times n_C} \\ 
\cdots & \cdots & \cdots & \cdots\\
0_{n_C \times n_1} & 0_{n_C \times n_2} & \cdots & W_C\end{array}\right),\]
where $0_{n \times m}$ denotes an $n\times m$ matrix with 0 elements, $\sum_{c=1}^C n_c= n$, and each matrix $W_c$ has size $n_c \times n_c$ and is irreducible~\cite{meyer01}[p.~671]. Each matrix corresponds to a connected component of the communication graph.

The eigenvalues of $W$ (taken with their multiplicity) are the eigenvalues of the different matrices $W_c$, for $c=1, \dots, C$. It is then sufficient to prove that the eigenvalues of each $W_c$ are in $(-1,1]$.
This result follows immediately from the fact that $W_c$ is  irreducible with non-negative elements on the diagonal, then it is primitive~\cite{meyer01}[Example 8.3.3], i.e., $1$ is the only eigenvalue on the unit circle. 
\end{proof}

The objective of the proof is to control the term $\sum_{t=0}^{T-1}\eta_t\frac{1}{m}\sum_{j=1}^{m}\EE[\|\nabla R_{S_j}(\theta_j^{(t)})\|_2]$ in Lemma \ref{lemma:link-local-optim}. It starts with the following descent lemma.

\begin{lemma}\emph{(Descent Lemma).} \label{lemma:descent} Let the same setting as Theorem \ref{thme:convex}, with a constant stepsize $\eta>0$ and additional Assumption \ref{hyp:bounded-variance}. We have:
        \begin{equation*}
            \frac{\eta}{m}\sum_{j=1}^{m}\EE[\|\nabla R_{S_j}(\theta_j^{(t)})\|_2^2] \leq \frac{2}{m}\sum_{j=0}^{m}\EE\Big[R_{S_j}(\theta_j^{(t)})-R_{S_j}(\theta_j^{(t+1)})\Big] + \beta\sigma^2\eta^2 + \frac{1}{m\eta}\sum_{j=1}^m\EE\|\sum_{l=1}^mW_{jl}\theta_l^{(t)}-\theta_j^{(t)}\|_2^2 
        \end{equation*}
    \end{lemma}

\begin{proof}
    For all $j=1,\ldots,m$, the convexity and $\beta$-smoothness of $R_{S_j}$ gives:
    \begin{align}
        R_{S_j} & (\theta_j^{(t+1)}) - R_{S_j}(\theta_j^{(t)}) \leq \langle \nabla R_{S_j}(\theta_j^{(t)}), \theta_j^{(t+1)} - \theta_j^{(t)} \rangle + \frac{\beta}{2}\|\theta_j^{(t+1)} - \theta_j^{(t)}\|_2^2 \nonumber \\
        & \leq \Big\langle \nabla R_{S_j}(\theta_j^{(t)}), \sum^m_{l=1}W_{jl}\theta_l^{(t)} - \eta\nabla \ell(\theta^{(t)}_j;Z_{I^t_jj}) - \theta_j^{(t)} \Big\rangle + \frac{\beta}{2}\Big\|\sum^m_{l=1}W_{jl}\theta_l^{(t)} - \eta\nabla \ell(\theta^{(t)}_j;Z_{I^t_jj}) - \theta_j^{(t)}\Big\|_2^2 \nonumber \\
        & \leq -\eta\Big\langle \nabla R_{S_j}(\theta_j^{(t)}), \nabla \ell(\theta^{(t)}_j;Z_{I^t_jj}) \Big\rangle + \Big\langle \nabla R_{S_j}(\theta_j^{(t)}), \sum^m_{l=1}W_{jl}\theta_l^{(t)} - \theta_j^{(t)} \Big\rangle \nonumber \\
        & \qquad + \frac{\beta}{2}\Big(\Big\|\sum^m_{l=1}W_{jl}\theta_l^{(t)} - \theta_j^{(t)}\Big\|_2^2 - 2\eta\Big\langle \nabla \ell(\theta^{(t)}_j;Z_{I^t_jj}), \sum^m_{l=1}W_{jl}\theta_l^{(t)} - \theta_j^{(t)} \Big\rangle  + \eta^2\Big\| \nabla \ell(\theta^{(t)}_j;Z_{I^t_jj}) \Big\|_2^2\Big) \label{eq:step-1-descent}
    \end{align}

Taking the conditional expectation of \eqref{eq:step-1-descent} given $\mathcal{F}_t$, the natural filtration at time $t$, and the dataset $S$ gives:

\begin{align}
    \EE[R_{S_j} (\theta_j^{(t+1)}) - R_{S_j}(\theta_j^{(t)})|\mathcal{F}_t, S] & \leq -\eta\Big\| \nabla R_{S_j}(\theta^{(t)}_j) \Big\|_2^2 + (1-\beta\eta)\Big\langle \nabla R_{S_j}(\theta_j^{(t)}), \sum^m_{l=1}W_{jl}\theta_l^{(t)} - \theta_j^{(t)} \Big\rangle \nonumber \\
    & \qquad \qquad + \frac{\beta}{2}\Big\|\sum^m_{l=1}W_{jl}\theta_l^{(t)} - \theta_j^{(t)}\Big\|_2^2 + \frac{\beta\eta^2}{2n}\sum_{i=1}^n\Big\| \nabla \ell(\theta^{(t)}_j;Z_{ij}) \Big\|_2^2  \nonumber \\
    & \leq \Big(\frac{\beta\eta^2}{2}-\eta\Big)\Big\| \nabla R_{S_j}(\theta^{(t)}_j) \Big\|_2^2 + \frac{\beta}{2}\Big\|\sum^m_{l=1}W_{jl}\theta_l^{(t)} - \theta_j^{(t)}\Big\|_2^2 + \frac{\beta\sigma^2\eta^2}{2} \nonumber \\
    & \qquad \qquad + (1-\beta\eta)\Big\langle \nabla R_{S_j}(\theta_j^{(t)}), \sum^m_{l=1}W_{jl}\theta_l^{(t)} - \theta_j^{(t)} \Big\rangle \label{eq:step-2-descent}
\end{align}

Applying the inequality $\langle a, b \rangle \leq \frac{\alpha}{2}\|a\|_2^2 + \frac{\alpha^{-1}}{2}\|b\|_2^2$, which is true $\forall \alpha>0$, to the last term in Equation \ref{eq:step-2-descent} gives, with $\alpha=\eta$:

\begin{align*}
    \EE[R_{S_j} (\theta_j^{(t+1)}) - R_{S_j}(\theta_j^{(t)})|\mathcal{F}_t, S] \leq  -\frac{\eta}{2} \Big\| \nabla R_{S_j}(\theta^{(t)}_j) \Big\|_2^2 +\frac{1}{2\eta}\Big\|\sum^m_{l=1}W_{jl}\theta_l^{(t)} - \theta_j^{(t)}\Big\|_2^2 + \frac{\beta\sigma^2\eta^2}{2}
\end{align*}

Passing to the expectation and rearranging the terms gives:

\begin{equation*}
    \eta \EE[\|\nabla R_{S_j}(\theta^{(t)}_j) \|_2^2] \leq 2\EE[R_{S_j} (\theta_j^{(t+1)}) - R_{S_j}(\theta_j^{(t)})] + \beta\sigma^2\eta^2 + \frac{1}{\eta}\EE[\|\sum^m_{l=1}W_{jl}\theta_l^{(t)} - \theta_j^{(t)}\|_2^2] \;.
\end{equation*}
Summing over $j=1,\ldots,m$ and dividing by $m$ provides the desired result.

\end{proof}

The following lemma controls the last term in Lemma \ref{lemma:descent}, under the setting of Theorem \ref{thme:optim-dep}.

    \begin{lemma}\emph{(Decentralization error control).} \label{lemma:C_W} Let the same setting as Theorem \ref{thme:convex}, with a constant stepsize $\eta$ and additional Assumption \ref{hyp:bounded-variance}. Assume further that $W$ is symmetric. Then, there exist a graph-dependent constant $C_W<\infty$ such that:

        \begin{equation*}
            \frac{1}{m\eta}\sum_{j=1}^m\EE[\|\sum_{l=1}^mW_{jl}\theta_l^{(t)}-\theta_j^{(t)}\|_2^2] \leq \eta L^2(C_W^{(t)})^2\;,
        \end{equation*}
        where $C_W^{(t)} \triangleq \sum_{s=0}^{t-1}\|W^s - W^{s+1}\|_2 \leq C_W$  and $\|\cdot\|_2$ is the $\ell_2$-operator norm.
        
    \end{lemma}

    \begin{proof}
        Let us rewrite the desired quantity in matrix form. Let $\Theta^{(t)}\in\IR^{m\times p}$ be the matrix that contains the parameters $\theta_1^{(t)},\ldots,\theta_m^{(t)}$ row-wise. In other word, the $j$-th row of $\Theta^{(t)}$ is $\theta_j^{(t)T}$. Similarly, let $\nabla L(\Theta^{(t)};Z_{I^t})\in\IR^{m\times p}$ be the matrix that contains the stochastic gradients $\nabla\ell(\theta_j^{(t)};Z_{I_j^tj})$, $j=1,\ldots,m$, also row-wise. When it is clear from context, we simply write $\nabla L(\Theta^{(t)})$. In matrix form, the quantity of interest is equal to $\frac{1}{m\eta}\EE\|W\Theta^{(t)} - \Theta^{(t)} \|_F^2$ and the D-SGD (variant B) updates are:

        \begin{align}
            \Theta^{(t)} & = W\Theta^{(t-1)} - \eta \nabla L(\Theta^{(t-1)};Z_{I^{t-1}}) \nonumber \\
            & = W^t\Theta^{(0)} - \eta \sum_{s=0}^{t-1}W^s\nabla L(\Theta^{(t-s-1)})\;. \nonumber
        \end{align}

        Hence, 

        \begin{align}
            \frac{1}{m\eta}\EE\|W\Theta^{(t)} - \Theta^{(t)} \|_F^2 & = \frac{1}{m\eta}\EE\|(I-W)\Theta^{(t)} \|_F^2 \nonumber \\
            & = \frac{1}{m\eta}\EE\Big\|(I-W) \Big(W^t\Theta^{(0)} - \eta \sum_{s=0}^{t-1}W^s\nabla L(\Theta^{(t-s-1)})\Big) \Big\|_F^2 \nonumber \\
            & = \frac{1}{m\eta}\EE\Big\|(W^t-W^{t+1})\Theta^{(0)} - \eta \sum_{s=0}^{t-1}(W^s- W^{s+1})\nabla L(\Theta^{(t-s-1)}) \Big\|_F^2 \nonumber \\
            & = \frac{\eta}{m}\EE\Big\| \sum_{s=0}^{t-1}(W^s- W^{s+1})\nabla L(\Theta^{(t-s-1)}) \Big\|_F^2 \label{eq:c_W_1} \;,
        \end{align}
        where we used the fact that $(W^t-W^{t+1})\Theta^{(0)} = (W^t-W^{t+1})(\Theta^{(0)} - \frac{\1\1^T}{m}\Theta^{(0)}) = \mathbf{0}$, since all agents start from the same initialization point $\theta^{(0)}$.

        Let's now control the quantity of interest without the square over the norm:

        \begin{align*}
            \Big\| \sum_{s=0}^{t-1}(W^s- W^{s+1})\nabla L(\Theta^{(t-s-1)}) \Big\|_F & \leq  \sum_{s=0}^{t-1} \Big\|(W^s- W^{s+1})\nabla L(\Theta^{(t-s-1)}) \Big\|_F \\
            & \leq \sum_{s=0}^{t-1} \|W^s- W^{s+1}\|_2\|\nabla L(\Theta^{(t-s-1)}) \|_F \\
            &  \hspace{-0.05cm} \overset{A. \ref{ass:lipschitz}}{\leq}  \hspace{-0.2cm} \sqrt{m}L\sum_{s=0}^{t-1} \|W^s- W^{s+1}\|_2 = \sqrt{m}LC_W^{(t)}\;
        \end{align*}

    Raising the last quantity to the square and plugging it into Equation \eqref{eq:c_W_1} gives the main result of Lemma \ref{lemma:C_W}. It remains to prove that $\exists C_W<\infty$ such that $C_W^{(t)} \leq C_W$.
    To this aim, we are going to show the sufficient condition that the series $C_W^{(\infty)} = \sum_{s=0}^{\infty}\|W^s - W^{s+1}\|_2$ converges to some limit $C_W<\infty$.

    Let $a_s\triangleq\|W^s - W^{s+1}\|_2$ be any term of the series and denote by $\lambda_1,\ldots,\lambda_m$, the eigenvalues of $W$, which are in $(-1,1]$ by Lemma \ref{lem:second_eigenvalue}. We therefore have $a_s=\sup_{k}\{|\lambda_k^s-\lambda_k^{s+1}|\} = \sup_{k}\{|\lambda_k|^s|1-\lambda_k|\}$. We note that if $\lambda_k=1$, $|\lambda_k|^s|1-\lambda_k| = 0$, so we can omit the eigenvalues equal to $1$ is the computation of $a_s$:

\begin{equation*}
    a_s = \sup_{k:\lambda_k\neq 1}\{|\lambda_k|^s|1-\lambda_k|\}\;.
\end{equation*}
If all eigenvalues are equal to $1$, $a_s=0$ for all $s\in \mathbb{N}$ and we directly have the convergence. Otherwise, we are going to show that the series converge by using the Cauchy root test, which states that a series converges if $\limsup_{s\rightarrow \infty}{|a_s|^\frac{1}{s}}=r<1$. As a matter of fact, we have

\begin{align*}
    |a_s|^\frac{1}{s} & = \sup_{k:\lambda_k\neq 1}\{|\lambda_k|^s|1-\lambda_k|\}^\frac{1}{s} \\
    & = \sup_{k:\lambda_k\neq 1}\{|\lambda_k||1-\lambda_k|^\frac{1}{s}\} \underset{s\rightarrow\infty}{\longrightarrow} \sup_{k:\lambda_k\neq 1}\{|\lambda_k|\}<1\;,
\end{align*}
which allows to conclude that the series converges and that there exists $C_W<\infty$ such that $C_W^{(t)} \leq C_W$.
        
\end{proof}

We can now prove Theorem \ref{thme:optim-dep}.Combining Lemma \ref{lemma:C_W} with Lemma \ref{lemma:descent}, we have:

\begin{align}
    \frac{\eta}{m}\sum_{j=1}^{m}\EE[\|\nabla R_{S_j}(\theta_j^{(t)})\|_2^2] & \leq \frac{2}{m}\sum_{j=0}^{m}\EE\Big[R_{S_j}(\theta_j^{(t)})-R_{S_j}(\theta_j^{(t+1)})\Big] + \beta\sigma^2\eta^2 + \eta L^2(C_W^{(t)})^2 \nonumber \\
    & \leq \frac{2}{m}\sum_{j=0}^{m}\EE\Big[R_{S_j}(\theta_j^{(t)})-R_{S_j}(\theta_j^{(t+1)})\Big] + \beta\sigma^2\eta^2 + \eta L^2C_W^2 \label{eq:thme-dep-step1}
\end{align}

Moreover, from Lemma \ref{lemma:link-local-optim} and with multiple use of Jensen inequality, we have

\begin{align}
    \frac{1}{m}\sum_{k=1}^m|\EE_{A,S}[R(A_k(S)) - R_S(A_k(S))]| & \leq \frac{2L\sigma \eta T}{mn} + \frac{2L}{mn}\sum_{t=0}^{T-1}\eta\frac{1}{m}\sum_{j=1}^{m}\EE[\|\nabla R_{S_j}(\theta_j^{(t)})\|_2] \nonumber \\
    & \leq \frac{2L\sigma \eta T}{mn} + \frac{2L\sqrt{T\eta}}{mn}\sqrt{\sum_{t=0}^{T-1}\eta\frac{1}{m}\sum_{j=1}^{m}\EE[\|\nabla R_{S_j}(\theta_j^{(t)})\|^2_2]} \label{eq:thme-dep-step2}
\end{align}

Summing Equation \eqref{eq:thme-dep-step1} over $t=0\ldots,T-1$ gives:

\begin{align*}
    \sum_{t=0}^{T-1}\eta\frac{1}{m}\sum_{j=1}^{m}\EE[\|\nabla R_{S_j}(\theta_j^{(t)})\|^2_2] & \leq \frac{2}{m}\sum_{j=0}^{m}\EE\Big[R_{S_j}(\theta_j^{(0)})-R_{S_j}(\theta_j^{(T)})\Big] + T\beta\sigma^2\eta^2 + T\eta L^2C_W^2 \\
    & \leq \frac{2}{m}\sum_{j=0}^{m}\EE\Big[R_{S_j}(\theta_j^{(0)})-R_{S_j}(\theta_{S_j}^*)\Big] + T\beta\sigma^2\eta^2 + T\eta L^2C_W^2
\end{align*}

Plugging this last equation into \eqref{eq:thme-dep-step2} gives the final result.

\hfill\qedsymbol{}

\section{Additional results and discussions}

%\subsection{Additional results}\label{app:additional-nonconvex}

\subsection{On the generalization of $A(S)=\bar{\theta}^{(T)}$}
\label{app:gen-avg}

In Remarks \ref{rmk:avg-worst-case}, we claimed that our generalization bounds are all also valid for the average of final iterates $A(S)=\bar{\theta}^{(T)}$. This is ensured by the following propositions.

\begin{proposition}\label{prop:gen-avg}
    Let $A(S)=\bar{\theta}^{(T)}$. Under the same set of hypotheses, the upper-bounds derived in Theorem \ref{thme:convex}, Theorem \ref{thme:strongly} and Theorem \ref{thme:optim-dep} are also valid upper-bounds on $|\EE_{A,S}[R(A(S)) - R_S(A(S))]|$.
\end{proposition}

\begin{proof}

    \underline{Extension of Theorem \ref{thme:convex}:} Like in the proof of the latter theorem, we are going to show that $A(S)$ is on-average $\varepsilon$-stable with $\varepsilon \leq \frac{2L \sum_{t=0}^{T-1}\eta_t}{mn}$.

    \begin{align*}
        \frac{1}{mn}\sum_{i,j}\EE[\|A(S)-A(S^{(i,j)})\|_2]& \leq \frac{1}{mn}\sum_{i,j}\frac{1}{m}\sum_{k=1}^m\EE[\|A_k(S)-A_k(S^{(i,j)})\|_2] \\
        &\leq \frac{1}{mn}\sum_{i,j}\frac{1}{m}\sum_{k=1}^m\EE[\delta_k^{(T)}(i,j)]\;,
 \end{align*}

where we took back the notation of the proof in \ref{app:convex}. Then based on Equation \eqref{eq:final-convex-worst} and the double stochasticity of $W$ and its powers, we get the desired result.

\underline{Extension of Theorem \ref{thme:strongly}:} In the same way, using this time Equation \eqref{eq:final-worst-strongly} we have:

\begin{equation*}
     \frac{1}{mn}\sum_{i,j}\EE[\|A(S)-A(S^{(i,j)})\|_2] \leq \frac{1}{mn}\sum_{i,j}\frac{1}{m}\sum_{k=1}^m\EE[\delta_k^{(T)}(i,j)] \leq \frac{4L}{\mu mn}\;,
\end{equation*}
which shows that $A(S)$ with $\mu$-strongly convex functions is on-average $\frac{4L}{\mu mn}$-stable.

\underline{Extension of Theorem \ref{thme:optim-dep}:} It suffices to show that Lemma \ref{lemma:link-local-optim} is also valid for $|\EE_{A,S}[R(A(S)) - R_S(A(S))]|$. Again we are going to use the link between generalization and on-average stability. This time, using Equation \eqref{eq:final-data-lemma} , we have:

\begin{align*}
     \frac{1}{mn}\sum_{i,j}\EE[\|A(S)-A(S^{(i,j)})\|_2] & \leq \frac{1}{mn}\sum_{i,j}\frac{1}{m}\sum_{k=1}^m\EE[\delta_k^{(T)}(i,j)] \\
     & \leq \frac{1}{mn}\sum_{i,j}\frac{1}{m}\sum_{k=1}^m\frac{2}{n}\sum_{t=0}^{T-1}(W^{T-t-1})_{kj}\eta_t\EE[\|\nabla\ell(\theta^{(t)}_j;Z_{ij})\|_2] \\
     & = \frac{2}{mn}\sum_{t=0}^{T-1}\eta_t\frac{1}{mn}\sum_{i=1}^{n}\sum_{j=1}^{m}\EE[\|\nabla\ell(\theta_j^{(t)};Z_{ij})\|_2]
\end{align*}

Using Lemma \ref{lemma:ob-avg-gen}, we therefore have 

\begin{equation*}
    |\EE_{A,S}[R(A(S)) - R_S(A(S))]| \leq \frac{2L}{mn}\sum_{t=0}^{T-1}\eta_t\frac{1}{mn}\sum_{i=1}^{n}\sum_{j=1}^{m}\EE[\|\nabla\ell(\theta_j^{(t)};Z_{ij})\|_2]\;.
\end{equation*}

We recognize the right-hand term of Equation \eqref{eq:nice-form} and using the same arguments, we finally show that 
\begin{equation*}
    |\EE_{A,S}[R(A(S)) - R_S(A(S))]| \leq \frac{2L\sigma}{mn}\sum_{t=0}^{T-1}\eta_t+ \frac{2L}{mn}\sum_{t=0}^{T-1}\eta_t\frac{1}{m}\sum_{j=1}^{m}\EE[\|\nabla R_{S_j}(\theta_j^{(t)})\|_2]\;.
\end{equation*}

\end{proof}

\begin{proposition}\label{prop:gen-avg-nonconv}
    Let $A(S)=\bar{\theta}^{(T)}$. Under the same set of hypotheses, the upper-bound derived in Theorem \ref{thme:non-convex} is also a valid upper-bound on $|\EE_{A,S}[R(A(S)) - R_S(A(S))]|$.
\end{proposition}

\begin{proof}

    Replacing $A_k$ by $A$ in the proof of Lemma \ref{lemma:key-non-conv}, and using the fact that $\ell\in [0,1]$ we get:

    \begin{align*}
        |\EE_{A,S}[R(A(S)) - R_S(A(S))]| & \leq  \frac{L}{mn}\sum_{i,j}\EE[\|A(S)  - A(S^{(ij)})\| \big| \mathcal{E}(i,j)]+ \frac{t_0}{n} \\
        & \leq \frac{L}{mn}\sum_{i,j}\frac{1}{m}\sum_{k=1}^m\EE[\|\theta_k^{(T)}  - \tilde{\theta}_k^{(T)}\| \big| \mathcal{E}(i,j)]+\frac{t_0}{n} \\
        & = \frac{1}{m}\sum_{k=1}^m \Delta_k^{(T)}+\frac{t_0}{n} \;
    \end{align*}

where $\Delta_k^{(T)}$ is defined in the proof of Theorem \ref{thme:non-convex} and can be controlled in the exact same way (Eq. \eqref{eq:recursion-non-convex}), leading to the same final result.

\end{proof}

\subsection{On the stepsize assumption in Theorem \ref{thme:convex}}
\label{app:step_bound}
In this section, we show that the assumption $\eta\leq \frac{2\min_k W_{kk}}{\beta}$ is rather mild, and automatically verified for typical choices of stepsize that ensure the convergence of D-SGD. More specifically, note that, when $W$ is symmetric, the iterates of Variant B are precisely the (stochastic) gradient steps for the optimization of the objective function
$$F(\Theta) = \sum_{k=1}^m \EE_{Z \sim \calD_k} [\ell(\theta_k; Z)] + \frac{1}{2\eta}\Theta^\top (I - W) \Theta\,,$$
where $\Theta\in\IR^{m\times d}$ is the concatenation of all local parameters. As this objective function is smooth and convex, typical convex optimization analysis requires the stepsize to be smaller than $1/\beta_F$, where $\beta_F > 0$ is the smoothness constant of $F$ (see \eg \citealp{bubeck2014convex} or \citealp[Theorem 3.4]{garrigos2023handbook}). However, a simple calculation shows that $\beta_F \leq \beta + \frac{1-\lambda_m}{\eta}$ (and this bound is tight as we have equality for the loss function $\ell(\theta,z)=\theta^2/2$), and the condition on the stepsize under our assumptions is thus $\eta \left( \beta + \frac{1-\lambda_m}{\eta} \right)\leq 1$, which gives
$$\eta \leq \frac{\lambda_m}{\beta}\,.$$
Finally, we conclude by noting that $\min_k W_{kk} = \min_k e_k^\top W e_k \geq \min_{u~:~\|u\|=1} u^\top W u = \lambda_m$, and thus the condition $\eta \leq \frac{\lambda_m}{\beta}$ directly implies $\eta \leq \frac{\min_k W_{kk}}{\beta}$ and the assumption of Theorem \ref{thme:convex}.

\subsection{Mistake in \citet{sun2021stability}}
\label{app:mistake_sun}

In Table \ref{tab:a} and Section \ref{sec:non-convex}, we claim that there is unfortunately a mistake in the proof of Theorem 3 in \citet{sun2021stability}, where the authors provide their generalization upper-bound for non-convex functions. In the paper, they provide an upper bound of order $\calO\big(T^{\frac{\beta c}{\beta c +1}}(\frac{1}{mn} + C_\rho )\big)$, however, here we suggest that they should have a bound of order $\calO\big(T^{\frac{\beta c}{\beta c +1}}(\frac{1}{n} + C_\rho )\big)$ instead.

Let's start by determining which part of the proof is wrong. To do this, we take the most recent version of the article on Arxiv as a reference, which can be found at: \url{https://arxiv.org/pdf/2102.01302.pdf}. 
The proof of Theorem 3 can be found on page 15, it relies on Lemma 7 which can be found on page 11. This Lemma is key in their proof and analogue to our Lemma \ref{lemma:key-non-conv} with a uniform rather than an on-average model stability argument. With our notation, it states that $\forall z \in Z$, $S$ and $S'$ that differ in a single data point, they have:

\begin{equation*}
        \EE[|\ell(\bar{\theta}^{(T)};z) - \ell(\bar{\tilde{\theta}}^{(T)};z)|] \leq \frac{t_0}{n}\sup_{\theta, z} \ell(\theta;z) + L\EE\Big[\|\bar{\theta}_k^{(T)} - \bar{\tilde{\theta}}^{(T)}\|\Big| \|\bar{\theta}^{(t_0)} - \bar{\tilde{\theta}}^{(t_0)}\| = 0\Big]\;.     
\end{equation*}
At this point, we observe that their Lemma is very similar to ours and notably that the first term, like us, is divided by $n$. The rest of the proof essentially consists of controlling the second term in the upper bound above. They show (page 16) that:

\begin{equation*}
    \EE\Big[\|\bar{\theta}_k^{(T)} - \bar{\tilde{\theta}}^{(T)}\|\Big| \|\bar{\theta}^{(t_0)} - \bar{\tilde{\theta}}^{(t_0)}\| = 0\Big] \leq \Big(\frac{2Lc}{mn} + 4(1+cL)LC_\rho\Big)c\beta\left(\frac{T}{t_0}\right)^{c\beta}\;.
\end{equation*}
However, when they plug this last equation into the one coming from Lemma 7 above, they claim that they have:

\begin{equation}
        \EE[|\ell(\bar{\theta}^{(T)};z) - \ell(\bar{\tilde{\theta}}^{(T)};z)|] \leq \frac{t_0}{mn} + L\Big(\frac{2Lc}{mn} + 4(1+cL)LC_\rho\Big)c\beta\left(\frac{T}{t_0}\right)^{c\beta}\;. \label{eq:mistake_sun}    
\end{equation}
Here, we can see that for no specific reason, \textbf{$t_0$ is now divided by $mn$ instead of $n$ only}, which is a mistake. This obviously increases stability, and taking $t_0=c^\frac{1}{c \beta +1}T^\frac{c \beta}{c \beta +1}$ leads to their final result. However, by taking this value for $t_0$ and dividing the first term by $n$ only would give a rate of order $\calO\big(T^{\frac{\beta c}{\beta c +1}}(\frac{1}{n} + C_\rho )\big)$ instead. We could imagine to take a better value for $t_0$ in the corrected version of Eq. \eqref{eq:mistake_sun}. However, this is out of the scope of this paper and and in any case, the resulting bound would be worse than ours due to the additional term in $C_\rho$.

\subsection{Experimental setup}
\label{app:exps}

%\textbf{Experimental setup.}
We consider a logistic regression problem with two classes. Each data point $(X,Y)$ is i.i.d. and drawn as follows. With probability $0.5$, the point is first associated to a class $Y=0$ or $1$. If $Y=0$ then $X$ follows a bivariate random Gaussian variable with vector mean $(1,-1)$ and isotropic covariance $I$. If $Y=1$, then the vector mean is $(-1,1)$. To make the problem slightly more complicated and avoid separability, $Y$ is then flipped with probability $0.1$. We take the loss $\ell$ to be:

\begin{equation*}
    \ell(\theta;(x,y)) = -y\log{\Big(\frac{1}{1+\exp{(-x^T\theta})}\Big)} - (1-y)\log{\Big(\frac{\exp{(-x^T\theta})}{1+\exp{(-x^T\theta})}\Big)}\;.
\end{equation*}

For the training, we have $m=20$ agents. To simulate the low noise regime, we take $n=1$ local data point (i.e. full batch: $\sigma^2 = 0$), while we take $n=10$ local data points in the higher noise regime. We then run D-SGD (Variant~B) for $T=500$ iterations, with constant step size $\eta=0.03$ and initial point $\theta^{(0)} = \mathbf{0}$. We consider four communication graphs: (i) Complete graph with uniform weights $1/m$, (ii) Identity graph $I$ (local SGD), (iii) Circle graph with self-edges and uniform weighs $1/3$, and (iv) Complete graph with diagonal elements equal to $0.95$ and remaining elements uniformly equal to $0.05/(m-1)$.

At each iteration $t=1,\ldots,T$, we compute a test loss (empirical population risk) using $500$ i.i.d. data points, evaluated at all parameters $\theta_1^{(t)},\ldots,\theta_m^{(t)}$, and compute the difference with the associated training loss (full empirical risk). Theses differences empirically correspond to local generalization errors, we then average them to match with Theorem \ref{thme:optim-dep}.

We repeat the experiment over $50$ different training-test data sets and for each data set, we run the algorithm $3$ times, which leads to $150$ different runs of D-SGD for each communication graphs. At the end, we average all these $150$ runs and take the absolute value.

\end{document}